%% file: 00-mainStabilitySO3.tex
\documentclass[review]{elsarticle}

\usepackage{lineno,hyperref}
\modulolinenumbers[5]

\usepackage[english]{babel}
\usepackage{booktabs}

\usepackage{ifpdf}

\usepackage{url}
\usepackage{hyperref}

\usepackage{color}

\usepackage{pgf, tikz, pgfplots}
\usetikzlibrary{shapes, arrows, automata}
\usetikzlibrary{calc,hobby,decorations}

\usepackage[cmex10]{amsmath}
\usepackage{amsfonts, amssymb, amsthm}
\usepackage{mathrsfs}

\usepackage{array}

\usepackage{booktabs}

\usepackage{algorithm} 
\usepackage{algorithmic}
\usepackage{enumerate}
\usepackage{multirow}
\usepackage{rotating}
\usepackage{subcaption}
\input{mySections.tex}

\input{mySymbol.sty}
\input{pennColors.sty}

\definecolor{penndarkestblue}{cmyk}{1,0.74,0,0.77}
\definecolor{penndarkerblue}{cmyk}{1,0.74,0,0.70}
\definecolor{pennblue}{cmyk}{0.99,0.66,0,0.57} 
\definecolor{pennlighterblue}{cmyk}{0.98,0.44,0,0.35}
\definecolor{pennlightestblue}{cmyk}{0.38,0.17,0,0.17} 

\definecolor{penndarkestred}{cmyk}{0,1,0.89,0.66}
\definecolor{penndarkerred}{cmyk}{0,1,0.88,0.55}
\definecolor{pennred}{cmyk}{0,1,0.83,0.42} 
\definecolor{pennlighterred}{cmyk}{0,1,0.6,0.24}
\definecolor{pennlightestred}{cmyk}{0,0.43,0.26,0.12} 

\definecolor{penndarkestgreen}{cmyk}{1,0,1,0.68}
\definecolor{penndarkergreen}{cmyk}{1,0,1,0.57}
\definecolor{penngreen}{cmyk}{1,0,1,0.44} 
\definecolor{pennlightergreen}{cmyk}{1,0,1,0.25}
\definecolor{pennlightestgreen}{cmyk}{0.43,0,0.43,0.13}

\definecolor{penndarkestorange}{cmyk}{0,0.65,1,0.49}
\definecolor{penndarkerorange}{cmyk}{0,0.65,1,0.33}
\definecolor{pennorange}{cmyk}{0,0.54,1,0.24} 
\definecolor{pennlighterorange}{cmyk}{0,0.32,1,0.13}
\definecolor{pennlightestorange}{cmyk}{0,0.15,0.46,0.06}
	
\definecolor{penndarkestpurple}{cmyk}{0,1,0.11,0.86}
\definecolor{penndarkerpurple}{cmyk}{0,1,0.13,0.82}
\definecolor{pennpurple}{cmyk}{0,1,0.11,0.71} 
\definecolor{pennlighterpurple}{cmyk}{0,1,0.05,0.46}
\definecolor{pennlightestpurple}{cmyk}{0,0.35,0.02,0.23}
	
\definecolor{pennyellow}{cmyk}{0,0.20,1,0.05} 
\definecolor{pennlightgray1}{cmyk}{0,0,0,0.05}
\definecolor{pennlightgray3}{cmyk}{0.01,0.01,0,0.18}
\definecolor{pennmediumgray1}{cmyk}{0.04,0.03,0,0.31}
\definecolor{pennmediumgray4}{cmyk}{0.08,0.06,0,0.54}
\definecolor{penndarkgray2}{cmyk}{0.09,0.07,0,0.71}
\definecolor{penndarkgray4}{cmyk}{0.1,0.1,0,0.92}

\def\SO3{\mathrm{SO(3)}}

\newtheorem{lemma}{\hspace{0pt}\bf Lemma}
\newtheorem{proposition}{\hspace{0pt}\bf Proposition}

\newtheorem{theorem}{\hspace{0pt}\bf Theorem}
\newtheorem{corollary}{\hspace{0pt}\bf Corollary}

\newtheorem{remark}{\hspace{0pt}\bf Remark}

\newtheorem{definition}{\hspace{0pt}\bf Definition}

\journal{Journal of Signal Processing}

\begin{document}

\begin{frontmatter}

\title{Spherical Convolutional Neural Networks:\\Stability to Perturbations in $\SO3$}

\author{Zhan Gao$^{\dagger}$\fnref{}, Fernando Gama$^{\ddagger }$ and Alejandro Ribeiro$^{\dagger}$}
\fntext[]{$^{ \dagger}$Department of Electrical and Systems Engineering, University of Pennsylvania, USA. Email: $\{$gaozhan,aribeiro$\}$@seas.upenn.edu. $^{\ddagger }$Department of Electrical Engineering and Computer Sciences, University of California, Berkeley, USA. Email: fgama@berkeley.edu}

\begin{abstract}
Spherical convolutional neural networks (Spherical CNNs) learn nonlinear representations from $3$D data by exploiting the data structure and have shown promising performance in shape analysis, object classification, and planning among others. This paper investigates the properties that Spherical CNNs exhibit as they pertain to the rotational structure inherent in spherical signals. We build upon the rotation equivariance of spherical convolutions to show that Spherical CNNs are stable to general structure perturbations. In particular, we model arbitrary structure perturbations as diffeomorphism perturbations, and define the rotation distance that measures how far from rotations these perturbations are. We prove that the output change of a Spherical CNN induced by the diffeomorphism perturbation is bounded proportionally by the perturbation size under the rotation distance. This stability property coupled with the rotation equivariance provide theoretical guarantees that underpin the practical observations that Spherical CNNs exploit the rotational structure, maintain performance under structure perturbations that are close to rotations, and offer good generalization and faster learning.
\end{abstract}

\begin{keyword}
Spherical convolutional neural networks, spherical convolutional filters, structure perturbations, stability analysis
\end{keyword}

\end{frontmatter}



\section{Introduction} \label{sec:intro}

\input{01-introductionStabilitySO3.tex}

\section{Spherical Convolutional Neural Networks} \label{sec:SCNN}

\input{02-SCNN.tex}


\section{Stability of Spherical Convolutional Filters} \label{sec:structureConv}

\input{03-structureConv.tex}


\section{Stability of Spherical Convolutional Neural Networks} \label{sec:stabilitySCNN}

\input{04-stabilitySCNN.tex}


\section{Numerical Experiments} \label{sec:sims}

\input{05-experimentsStabilitySO3.tex}


\section{Conclusions} \label{sec:conclusions}

\input{06-conclusionsStabilitySO3.tex}


\appendix 

\input{A1-proofsStabilitySO3.tex}


\bibliography{bibFiles/myIEEEabrv,bibFiles/biblioStabilitySO3}

\end{document}

%% file: mySections.tex
\usepackage{needspace}

\newcommand{\myparagraph}[1]{\needspace{1\baselineskip}\medskip\noindent {\bf #1}}



%% file: 01-introductionStabilitySO3.tex



Modern image-acquisition technologies such as light detection and ranging (LIDAR) \cite{reutebuch2005light}, panorama cameras \cite{Chang2017Matterport3D} and optical scanners \cite{dury2015surface} obtain data that can be modelled on the spherical surface in $3$D space \cite{zhou20043d}. Spherical signals offer mathematical representations for such data, which essentially assign a scalar (or vector) value to each point on the spherical surface \cite{simons2013scalar, racah2017extremeweather}, and have been leveraged in applications of object classification in computer vision \cite{yavartanoo2018spnet}, panoramic video processing in self-driving cars \cite{geiger2013vision} and $3$D surface reconstruction in medical imaging \cite{maier2013optical}. Processing spherical signals in a successful manner entails architectures capable of exploiting the structural information given by the spherical surface. In particular, we are concerned about exploiting the rotational structure inherent in spherical signals \cite{makadia2004rotation, makadia2006rotation}.   

The rotational structure can be captured by means of the rotation group. This is a mathematical group defined in the space of spherical surface equipped with the operation of rotation \cite{gelfand2018representations}. The rotation group admits the definition of convolution operation for spherical signals, which in turn gives rise to spherical convolutional filters \cite{derighetti2011convolution}. The latter are linear processing operators that compute weighted, rotated linear combinations of spherical signal values \cite{makadia2010spherical, su2017learning}. Spherical convolutional filters exhibit the property of rotation equivariance, which means that they yield the same features irrespective of rotated versions of the input spherical signal \cite{esteves2018learning,cohen2018spherical}. This property is particularly meaningful when processing $3$D data where the information is contained in the relative location of signal values but not their absolute positions \cite{worrall2017harmonic,veeling2018rotation,li2019discrete}. Therefore, processing spherical signals with spherical convolutional filers, allows the resulting architecture to exploit the data structure for information extraction.

Spherical convolutional neural networks (Spherical CNNs) are developed as nonlinear processing architectures that consist of a cascade of layers, each of which applies a spherical convolutional filter followed by a pointwise nonlinearity \cite{esteves2018learning, cohen2018spherical, khasanova2017graph, frossard2017graph, perraudin2019deepsphere, defferrard2020deepsphere}. The inclusion of nonlinearities coupled with the cascade of multiple layers dons Spherical CNNs with an enhanced representative power, thus exhibiting superior performance for spherical signal processing. Several Spherical CNN architectures have been proposed, leveraging different implementations of the spherical convolution. In particular, \cite{esteves2018learning, cohen2018spherical} carry out the spherical convolution through the spherical harmonic transform, while \cite{khasanova2017graph, frossard2017graph, perraudin2019deepsphere, defferrard2020deepsphere} discretize the sphere using a graph connecting pixels and then approximate the spherical convolution with the Laplacian-based graph convolution.


Considering the evident success of Spherical CNNs in $3$D tasks, we focus on analyzing the properties that Spherical CNNs exhibit as they pertain to the rotational structure present in spherical signals. These properties are meant to shed light on the reasons behind this observed success and provide theoretical guarantees on the performance robustness under structure perturbations. Rotation equivariance of spherical convolutional linear filters was established in \cite{esteves2018learning, cohen2018spherical}, and Spherical CNNs were shown to be numerically effective in standard retrieval and classification problems.


In this paper, we further investigate how spherical convolutional filters and Spherical CNNs react to general structure perturbations applied on input spherical signals. More specifically, we build upon the rotation equivariance of spherical convolutions \cite{esteves2018learning, cohen2018spherical} to prove that Spherical CNNs are Lipschitz stable to structure perturbations. To conduct such analysis, we leverage the Euler parametrization together with the normalized Haar measure to carry out spherical convolutions explicitly in the spherical coordinate system. We also model the arbitrary structure perturbation as the diffeomorphism perturbation, and define a notion of rotation distance that measures the difference between diffeomorphism perturbations and rotation operations. The established stability results indicate that Spherical CNNs extract the same information under rotations, and more importantly, it is able to maintain performance when structure perturbations are close to rotations. Our detailed contributions can be summarized as follows.

\smallskip
\begin{enumerate}[(i)]


\item \emph{Stability of spherical convolutional filters (Section \ref{sec:structureConv})}: We prove the output difference of spherical convolutional filters induced by general structure perturbations, defined as diffeomorphism perturbations, is bounded proportionally by the perturbation size under the rotation distance [Thm. \ref{thm:stabilityConv}]. This result extends the stability analysis from regular rotation operations to irregular diffeomorphism perturbations, and establishes that almost same information can be extracted if structure perturbations are close to rotation operations. 

\item \emph{Stability of Spherical CNNs (Section \ref{sec:stabilitySCNN})}: We show Spherical CNNs inherit the stability to general structure perturbations with a factor proportional to the perturbation size [Thm. \ref{thm:stabilitySCNN}]. It also indicates the explicit impact of the filter, the nonlinearity, the architecture width and depth on the stability property. In particular, a wider and deeper Spherical CNN degrades the stability but improves the representative power, indicating a trade-off between these two factors.

\end{enumerate}

The stability property to arbitrary perturbations plays an important role in improving the generalization capability of Spherical CNNs. In the case of $3$D object identification \cite{wu20153d}, for instance, mild changes in spherical signals may be introduced by different viewing angles or distances but Spherical CNNs are expected to extract the same information. Section \ref{sec:sims} provides numerical experiments to corroborate theory on the problem of $3$D object classification and Section \ref{sec:conclusions} concludes the paper. All proofs are collected in the appendix.

%% file: 02-SCNN.tex


Signals arising from $3$D data, e.g., LIDAR images, X-ray models, surface data, etc., can be described as inscribed on the spherical surface in $3$D space \cite{simons2013scalar, racah2017extremeweather}. In Sec.~\ref{subsec:sphericalSignals} we introduce the mathematical description of these spherical signals, in Sec.~\ref{subsec:sphericalConv} we discuss the spherical convolution as the fundamental linear operation between spherical signals, and in Sec.~\ref{subsec:sphericalCNN} we present the spherical convolutional neural network (Spherical CNN).

%
\begin{figure}
    \centering
    \includegraphics[width=0.75\columnwidth]
    {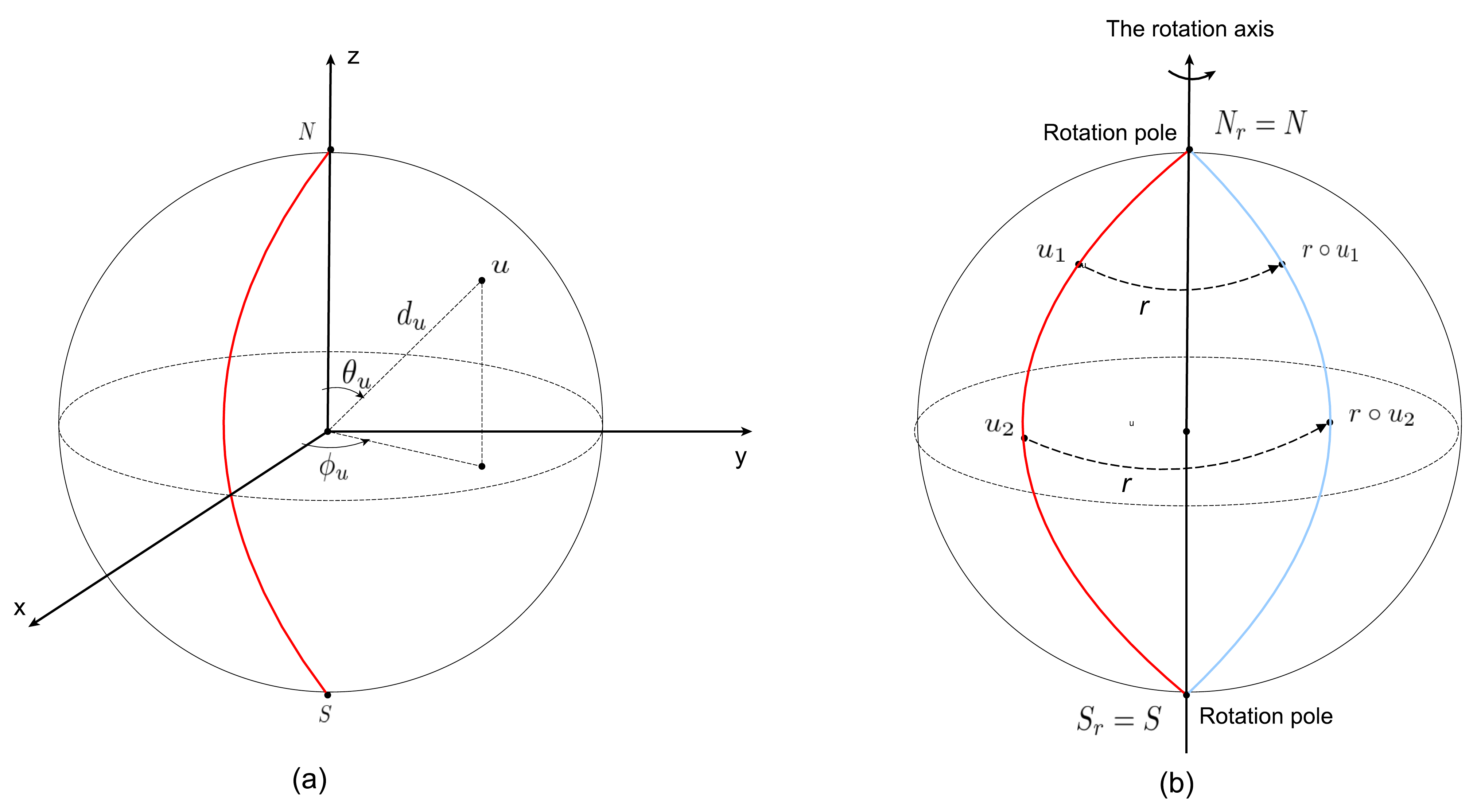}
    \caption{(a) The spherical coordinate system. (b) Azimuthal rotation on the sphere. The rotation displaces points on the sphere along the rotation axis, e.g., it displaces points on the red line to corresponding points on the blue line.}%
    \label{fig:sphericalPoint}
\end{figure}
%


\subsection{Spherical signal} \label{subsec:sphericalSignals}

Let $\mbS_{2} \subset \reals^{3}$ be the spherical surface contained in $\reals^{3}$. Given a coordinate system $\ccalS(x,y,z)$, a point $u = (x_{u}, y_{u}, z_{u}) \in \mbS_{2}$ can be alternatively described by the polar angle $\theta_u = \arctan(\sqrt{x_{u}^{2}+y_{u}^{2}}/z_{u}) \in [0,\pi]$ from the $z$-direction, the azimuth angle $\phi_u = \arctan(x_{u}/y_{u}) \in [0,2\pi)$ along the $xy$-plane, and the distance $d = \sqrt{x_u^{2}+y_u^{2}+z_u^{2}}$ from the origin. Without loss of generality, we assume the unit sphere with $d=1$---see Fig. 1a for the spherical coordinate system. More specifically, a point $u \in \mbS_{2}$ can be represented by the vector
\begin{equation} \label{eq:sphericalPoint}
    \!u 
    \!=\! (\theta_u,\phi_u) 
    \!=\!\! \big[
        \sin(\theta_u) \cos(\phi_u),
        \sin(\theta_u) \sin(\phi_u),
        \cos(\theta_u)
    \big].
\end{equation}
For points whose polar angle are $\theta_u = 0$ or $\theta_u = \pi$, the azimuth angle is assumed to be zero. These points are known as the north and south poles, referred to as $N \in \mbS_{2}$ and $S \in \mbS_{2}$ respectively [cf. Fig.~\ref{fig:sphericalPoint}a].

A \emph{spherical signal} is defined as the map $x: \mbS_{2}\to \reals$ which assigns a scalar $x(u) \in \reals$ to each point $u \in \mbS_{2}$ on the sphere. Equivalently, these signals can be represented as mappings from the pair of angular variables $(\theta_u,\phi_u)$ to the real line $\mbR$, i.e. $x(u) = x(\theta_u,\phi_u) \in \mathbb{R}$. Spherical signals are typically used to describe data collected from $3$D objects---see Fig.~\ref{fig:tetrahedron} for an example on how to cast the surface of $3$D objects as spherical signals.

%
\begin{figure}[t]
    \centering
    \includegraphics[width=0.5\columnwidth, trim=10 10 10 10]{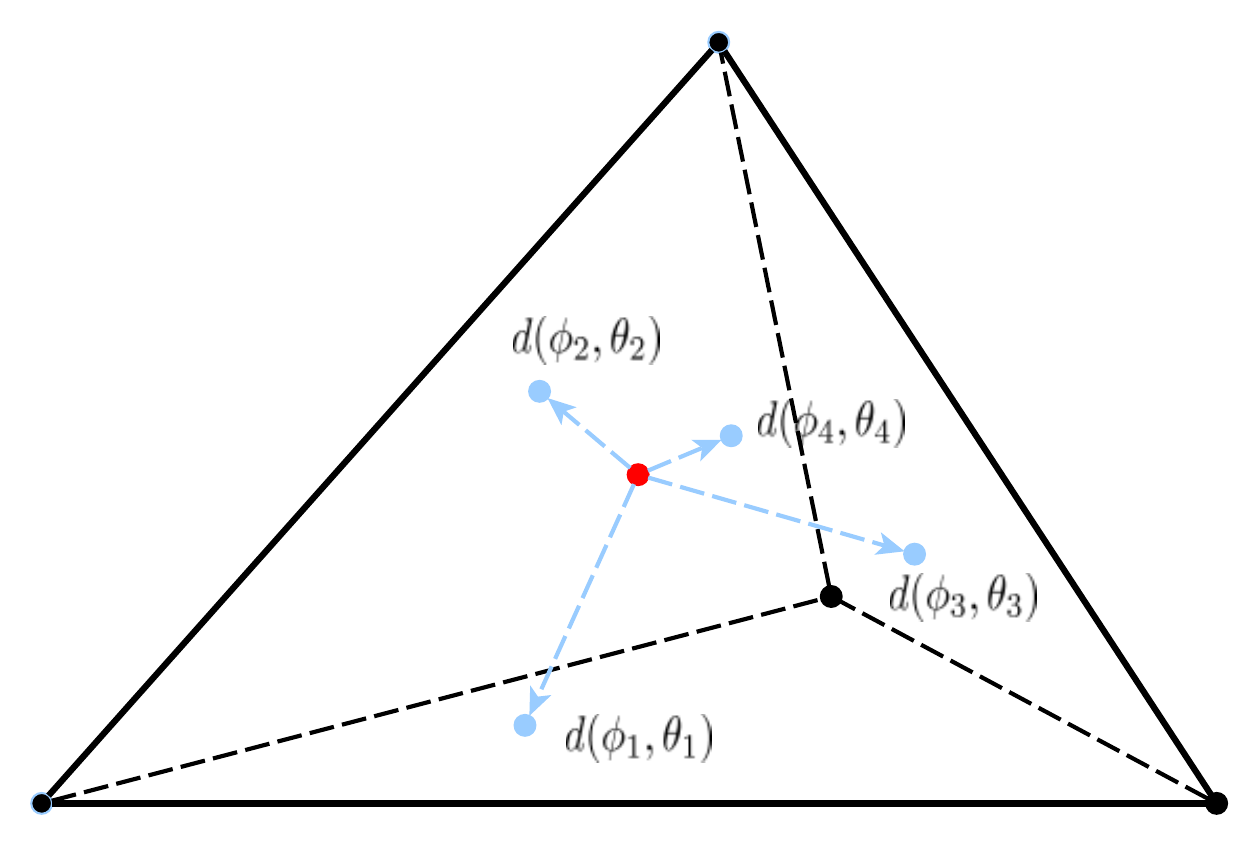}
    \caption{Spherical signal representing the surface data of a tetrahedron. The tetrahedron contains interior points from which the whole boundary is visible. Assume the object center as one of these interior points, from which we cast rays that intersect the object surface. Denote by $d(\phi_u,\theta_u)$ the distance between the farthest intersection point and the center along the ray direction $(\phi_u, \theta_u)$, where $\theta_u$ is the polar angle and $\phi_u$ is the azimuth angle. The surface data can be represented as a spherical signal $x(\phi_u, \theta_u) = d (\phi_u, \theta_u) \in \mathbb{R}$, which maps angular variables to the distance.}
    \label{fig:tetrahedron}
\end{figure}
%


\subsection{Rotation operation} \label{subsec:rotation}

Rotations are elementary operations to which a spherical signal can be subject \cite{makadia2004rotation,makadia2006rotation}. A \emph{rotation} $r$ is uniquely characterized by the rotation point $N^{r} \in \mbS_{2}$ and the rotation angle $\beta^{r} \in [0,2\pi)$. It displaces the points on the sphere by $\beta^{r}$ degrees along the axis that passes through the origin and the point $N^{r}$; thus, we may refer to $N^{r}$ as the rotation axis. To describe the rotation in the spherical coordinate system [cf. \eqref{eq:sphericalPoint}], let $r^{\theta}: \mbS_{2}\to \reals$ be the polar angle displacement and $r^{\phi}: \mbS_{2}\to \reals$ be the azimuth angle displacement induced by the rotation $r$. These two quantities parametrize the rotation operation $r:\mbS_{2} \to \mbS_{2}$ as
\begin{align} \label{eq:rotationOp}
    r (u)
    & = r (\theta_u,\phi_u) = \!\big[
        \sin(\theta_u + r^{\theta}(\theta_u,\phi_u)) \cos(\phi_u + r^{\phi}(\theta_u,\phi_u)),\ \\
    &\qquad\qquad\qquad \sin(\theta_u \!+\! r^{\theta}(\theta_u,\phi_u)) \sin(\phi_u \!+\! r^{\phi}(\theta_u,\phi_u)),~\cos(\theta_u \!+\! r^{\theta}(\theta_u,\phi_u))
    \big]\nonumber
\end{align}
where $r^{\theta}$ and $r^{\phi}$ are determined by the Rodrigues' rotation formula \cite{mebius2007derivation}. The rotation operation takes a point $u$ on the sphere $\mbS_{2}$ and maps it into another point $r(u)$, also on the sphere $\mbS_{2}$, but where the polar angle has been displaced by $r^{\theta}$ and the azimuthal angle by $r^{\phi}$. The rotation axis intersects the spherical surface at two points, namely $N^{r} \in \mbS_{2}$ referred to as the rotation north pole and $S^{r} \in \mbS_{2}$ referred to as the rotation south pole. These two points remain unchanged under the rotation operation. As an illustrative example, consider the azimuthal rotation whose rotation axis coincides with the $z$-axis of the coordinate system $\ccalS(x,y,z)$, i.e. $N^{r} = N$ and $S^{r} = S$---see Fig.~\ref{fig:sphericalPoint}b. Then, $r^{\theta}(\theta_{u},\phi_{u}) = 0$ and $r^{\phi}(\theta_{u},\phi_{u}) = \beta^{r}$ are constants for all $(\theta_{u},\phi_{u}) \in [0,\pi]\times [0,2\pi)$.

We can also describe the rotation operation $r \in \SO3$ in terms of the $ZYZ$ Euler parametrization as
\begin{equation} \label{eq:rotationEuler}
r = r_{\phi_{r} \theta_{r} \rho_{r}} = r^{z}_{\phi_{r}} \circ r^{y}_{\theta_{r}} \circ r^{z}_{\rho_{r}}
\end{equation}
where $r_{\phi_{r}}^{z}$, $r_{\theta_{r}}^{y}$ and $r_{\rho_{r}}^{z}$ are the rotations along the Euclidean coordinate axes with the rotation angles of $\phi_{r} \in [0, 2\pi)$, $\theta_{r} \in [0,\pi]$ and $\rho_{r} \in [0, 2\pi)$, respectively. The Euler parametrization allows us to decompose any rotation $r$ as three consecutive rotations. The set of all rotation operations about the origin [cf. \eqref{eq:rotationOp}] defines a mathematical group under the operation of composition, referred to as the \emph{$3$D rotation group} $\SO3$ \cite{scott2012group, kostelec2008ffts}. This is a group in the sense that (i) composing two rotations $r_{1}, r_{2} \in \SO3$ results in another rotation $r_{1} \circ r_{2} \in \SO3$, (ii) rotations are associative $(r_{1} \circ r_{2}) \circ r_{3} = r_{1} \circ (r_{2} \circ r_{3})$ for any $r_{1},r_{2},r_{3} \in \SO3$, (iii) the rotation determined by $r^{\theta}=r^{\phi}=0$ is the identity rotation $r_{0} \in \SO3$ such that $r \circ r_0 = r$ for all $r \in \SO3$, and (iv) for every rotation $r \in \SO3$ there exists another rotation $r^{-1} \in \SO3$, such that $r \circ r^{-1} = r_{0}$.


\subsection{Spherical convolution} \label{subsec:sphericalConv}

The spherical surface $\mbS_{2}$ is now equipped with the rotation operation that conforms a mathematical group. This allows defining the convolution operation for spherical signals. Given two signals $x,h: \mbS_{2} \to \reals$ the spherical convolution $\ast_{\SO3}$ between them is defined as
\begin{equation} \label{eq:sphericalConv}
y(u) = (h \ast_{\SO3} x)(u) = \int_{\SO3} h\big(r^{-1} (u)\big) x\big(r(u_0)\big)\ dr,~\forall~ u \in \mbS_2
\end{equation}
where $y(u)$ is the output spherical signal and $u_0 = (0,0) \in \mbS_2$ is the point given by $\theta_{u_0} = \phi_{u_0} = 0$.

We can rewrite the convolution operation in the spherical coordinate system by leveraging the Euler parametrization of rotations [cf. \eqref{eq:rotationEuler}]. Recall the normalized Haar measure on the rotation group $\SO3$ as \cite{driscoll1994computing}
\begin{equation} \label{eq:normalizedMeasureRotation}
    dr = \frac{d\phi_{r}}{2 \pi} \frac{\sin(\theta_{r})d\theta_{r}}{2} \frac{d \rho_{r}}{2 \pi},
\end{equation}
and we can rewrite \eqref{eq:sphericalConv} as
\begin{align} \label{eq:sphericalConvEuler}
y(u)&= (h \ast_{\SO3} x)(u) \\
&= \frac{1}{8 \pi^2} \!\int\!
        h\big( r^{-1}_{\phi_{r}\theta_{r}\rho_{r}} (u) \big)  x\big( r_{\phi_{r}\theta_{r}\rho_{r}} (u_0)\big) \sin(\theta_{r})  \ d\theta_{r} d\phi_{r} d\rho_{r} \nonumber
\end{align}
where $r_{\phi_{r}\theta_{r}\rho_{r}}$ is the rotation parametrized by Euler angles $\phi_{r},\theta_{r},\rho_{r}$ [cf. \eqref{eq:rotationEuler}] and $r^{-1}_{\phi_{r}\theta_{r}\rho_{r}}$ is its inverse rotation. In the spherical coordinate system, by noting that $r_{\phi_{r}\theta_{r}\rho_{r}}(u_{0}) = (\theta_{r}, \phi_{r})$ and using the Rodrigues' formula \eqref{eq:rotationOp}, we further get
\begin{align} \label{eq:sphericalConvRectangle}
& y(\theta_{u},\phi_{u}) 
 = (h \ast_{\SO3} x)(\theta_{u},\phi_{u}) \\
& \!=\!\! \frac{1}{8 \pi^2} \!\!\int\!\!
h\big( \theta_{u}\!+\!r_{\phi_{r}\theta_{r}\rho_{r}}^{\theta \ -1}(\theta_{u},\phi_{u}),  \phi_{u}\!+\!r_{\phi_{r}\theta_{r}\rho_{r}}^{\phi \ -1}(\theta_{u},\phi_{u})\big) x( \theta_{r},\phi_{r} )\sin(\theta_{r})  \ d\theta_{r} d\phi_{r} d\rho_{r}.\nonumber
\end{align}
where $r_{\phi_{r}\theta_{r}\rho_{r}}^{\theta \ -1}$ and $r_{\phi_{r}\theta_{r}\rho_{r}}^{\phi \ -1}$ are angle displacements induced by $r_{\phi_{r}\theta_{r}\rho_{r}}^{ -1}$. We see that \eqref{eq:sphericalConvRectangle} is equivalent to \eqref{eq:sphericalConv} and \eqref{eq:sphericalConvEuler}, but described in terms of the polar angle $\theta_{u}$ and the azimuthal angle $\phi_{u}$ on which we can evaluate the output of spherical convolution mathematically. That is, we move from a description given entirely in the realm of rotation group $\SO3$ to a description given by specific angular variables in the spherical coordinate system [cf. \eqref{eq:sphericalPoint}]. It is important because this allows us to carry out the integral of the spherical convolution explicitly, serving for the stability analysis in the following sections.

%
\begin{remark}[Group convolution] \label{rmk:groupConv} \normalfont
Convolution operations can be formally defined with respect to any group that is applicable to the space on which the convolving signals are defined. That is, let $x,h:\mbX \to \reals$ be two signals mapping some space $\mbX$ into the real line, and let $\mathsf{G} = \mathsf{G}(\mbX, \circ)$ be a group where the element $g \in \mathsf{G}$ is defined over $\mbX$, $\circ$ is the binary operation, and $g^{-1} \in \mathsf{G}$ is the inverse element. The convolution operation $\ast_{\mathsf{G}}$ is generically defined as
\begin{equation} \label{eq:groupConv}
(h \ast_{\mathsf{G}} g)(u) = \int_{\mathsf{G}} h\big(g^{-1} (u) \big) x\big(g (u_{0})\big)\ dg
\end{equation}
with $u, u_{0} \in \mbX$. The spherical convolution \eqref{eq:sphericalConv} is one instance of the group convolution for $\mathsf{G} = \SO3$. Another example is the regular convolution defined over the group of translations. We remark that many of the results derived in this work are applicable to generic group convolutions---see \cite{Mallat12-Scattering} for details.
\end{remark}
%


\subsection{Spherical convolutional neural network} \label{subsec:sphericalCNN}

Spherical convolutions are linear operations that exploit the rotational structure present in spherical signals. Thus, they can be used to regularize the linear transform in neural networks such that the resulting architecture is able to leverage the rotational structure when processing spherical signals from $3$-D tasks. In particular, neural networks are nonlinear maps that consist of a cascade of $L$ layers, each applying a linear transform $A_{\ell}$ followed by a pointwise nonlinearity $\sigma_{\ell}$ \cite[Ch. 6]{Goodfellow16-DeepLearning}
\begin{equation} \label{eq:neuralNetwork}
    x_{\ell} = \sigma_{\ell}\Big[A_{\ell} x_{\ell-1} \Big] \ , \ \ell = 1,\ldots,L
\end{equation}
with $x_{0} = x$ the input signal. The number of layers $L$ as well as the specific form of nonlinearity $\sigma_{\ell}$ are typically design choices, while the linear trasforms $\{A_{\ell}\}_{\ell=1}^L$ are \emph{learned} by minimizing some objective function over a training set. Neural networks in such a general form work well only for small input data size, otherwise the space of learnable linear transforms $\{A_{\ell}\}_{\ell=1}^L$ becomes too large to efficiently explore, and the resulting mapping rarely generalizes well \cite[Ch. 9]{Goodfellow16-DeepLearning}. Successful neural network architectures like convolutional neural networks \cite{krizhevsky2012imagenet, hu2014convolutional, li2015convolutional} or graph neural networks \cite{Bruna14-DeepSpectralNetworks, Defferrard17-ChebNets, Gama19-Architectures} typically regularize the space of linear transforms by imposing the data structure information on the operation $A_{\ell}$ (e.g., CNNs exploit the Euclidean structure present in data and GNNs exploit the graph structure present in data).

Working with spherical signals demands neural network architectures that exploit the rotational structure. This motivates to develop the spherical convolutional neural network (Spherical CNN) where we regularize the linear transform $A_{\ell}$ to be a spherical convolution [cf. \eqref{eq:sphericalConv}], namely
\begin{equation} \label{eq:singleFeatureSCNN}
    x_{\ell} = \sigma_{\ell} \Big[ h_{\ell} \ast_{\SO3} x_{\ell-1} \Big] \ , \ \ell = 1,\ldots,L
\end{equation}
with $x_{0} = x$ the input signal \cite{esteves2018learning}. We see in \eqref{eq:singleFeatureSCNN} that the linear transform is now forced to be a spherical convolution. The learnable linear transforms are now the collection of spherical convolutional filters $\ccalH = \{h_{\ell}\}_{\ell=1}^L$. 

The descriptive power of Spherical CNNs can be enhanced by considering a multi-feature signal mapping $\bbx: \mbS_{2} \to \reals^{F}$, where we map each point on the sphere to a $F$-dimensional vector. Each entry of the vector is termed a \emph{feature} and the multi-feature spherical signal can be thought of as a collection of $F$ single-feature spherical signals $x^{f}: \mbS_{2} \to \reals$, i.e. $\bbx = \{x^{f}\}_{f=1}^F$. In this context, the convolution operation becomes a series of spherical convolutions where the input signal is processed through a filter bank. More specifically, an output signal with $G$ features $\bby: \mbS^{2} \to \reals^{G}$ can be obtained by
\begin{equation} \label{eq:filterBankConv}
    y^{g} = \sum_{f=1}^{F} h^{fg} \ast_{\SO3} x^{f},~\forall~f = 1,\ldots,F,~ g = 1,\ldots,G.
\end{equation}
Operations \eqref{eq:filterBankConv} define the spherical convolution acting on multi-feature spherical signals, which is equivalent to filtering $F$ features of the input signal through a bank of $FG$ filters and then aggregating the output of each $g$th filter across all $F$ input features. To simplify notation, we denote the single-feature spherical convolution [cf. \eqref{eq:sphericalConv}] as a filtering operation, i.e. $y = H(x) = h \ast_{\SO3} x$. The output of the last layer is $\Phi(x;\ccalH) = x_{L}$ where $\Phi:\mbS_{2} \to \mbS_{2}$ represents the nonlinear mapping given by the Spherical CNN [cf. \eqref{eq:singleFeatureSCNN}]. The learnable parameters in the Spherical CNN consist of the filters $\ccalH = \{H_{\ell}\}_{\ell=1}^L$ (or $\ccalH = \{H_{\ell}^{fg}\}_{\ell=1}^L$ for the multi-feature scenario). The design hyperparameters are the number of layers $L$, the nonlinearity $\sigma_{\ell}$, and the number of features $F_{\ell}$ at each layer.

The resulting Spherical CNN [cf. \eqref{eq:singleFeatureSCNN}] leverages the $3$D data structure and thus, we expect it to exhibit strong performance on $3$D tasks (as a matter of fact, this is observed in practice \cite{esteves2018learning}). In what follows, we aim to analyze the properties that Spherical CNNs exhibit when processing spherical signals, i.e., the stability to arbitrary structure perturbations, to explicitly illustrate how they exploit the rotational structure and to bring insights behind their observed superior performance.



%% file: 03-structureConv.tex


Adopting the spherical convolution as the main operation to process spherical signals exploits the rotational structure of $3$D data and thus exhibits improved performance. In this section, we analyze the effect that structure perturbations in spherical signals have on the output of spherical convolutional filters, to explain theoretical reasons behind its superior performance.



\subsection{Rotation equivariance}\label{subsec:rotationEquivarianceConv}

We begin by considering rotation operations as fundamental structure perturbations in spherical signals. Let $r \in \SO3$ be a rotation and $x$ be a spherical signal. Define the rotated signal $x_r$ as
\begin{equation} \label{eq:rotatedSignal}
    x_{r}(u) = x(r (u)) = x(\theta_u + r^{\theta}(\theta_u,\phi_u),\phi_u + r^{\phi}(\theta_u,\phi_u)),~\forall~ u \in \mbS_2,
\end{equation}
where the last equality responds to the rotation displacements following the Rodrigues' formula [cf. \eqref{eq:rotationOp}]. We formally state the rotation equivariance of spherical convolutional filters in the following proposition. 

%
\begin{proposition} \label{prop:rotationEquivarianceConv}
Let $x$ be a spherical signal, $H$ be a spherical convolutional filter, and $r \in \SO3$ be a rotation. Denote by $y=H(x)$ the filter output and $y_r$ the rotated signal of $y$ [cf. \eqref{eq:rotatedSignal}]. Then, it holds that
\begin{equation} \label{eq:rotationEquivarianceConv}
y_{r} = H(x_{r}).
\end{equation}
\end{proposition}
%
%
\begin{proof}
    See \ref{proof:PropositionEquivarianceCov} for a detailed proof that completes that in \cite{esteves2018learning, cohen2018spherical}.
\end{proof}
%
%

Proposition \ref{prop:rotationEquivarianceConv} states that applying a spherical convolutional filter $H$ to a rotated version of the signal $H(x_{r})$ yields an associated rotated version of the output $y_{r}$, where $y = H(x)$ is the output of filtering the original signal. This indicates that the spherical convolutional filter is capable of harnessing the same information irrespective of what rotated version of the spherical signal is processed. The rotation equivariance serves to improve transference and speed up training, since learning how to process one instance of the signal is equivalent to learning how to process all rotated versions of the same signal. This is analogous to the effect that translation equivariance has on CNNs \cite{Mallat12-Scattering} and the effect that permutation equivariance has on GNNs \cite{Gama20-Stability}.



\subsection{Signal dissimilarity modulo rotations} \label{subsec:signalDissimilarity}

In general, we are more interested in how Spherical CNNs fare when acting on two signals that are similar but not quite the same. In a sense, we expect the architecture to tell apart spherical signals when they are sufficiently different but to give a similar output if the signal change is due to noise or unimportant causes \cite{chen2011geometric, huang2016damped, nazari2017data}. The rotation equivariance (Prop.~\ref{prop:rotationEquivarianceConv}) states that the same information is obtained when processing all rotated versions of the same signal. This suggests that to measure how (dis)similar two signals are, we need to abstract rotations or, in other words, we need to do it modulo rotations. Motivated by this consideration, we define the rotation distance as follows.

%
\begin{definition}[Rotation distance] \label{def:rotationDistance}
For two spherical signals $x, \hhatx$, the rotation distance is defined as
\begin{equation} \label{eq:rotationDistance}
\| x-\hhatx \|_{\SO3} = \inf_{r \in \SO3} \| x - \hhatx_{r} \|
\end{equation}
where $\hhatx_{r}(u) = \hhatx(r (u))$ is the rotated spherical signal of $\hhatx$ by the rotation $r$ and $\| \cdot \|$ is the norm representation.
\end{definition}
%
%
\noindent Any valid norm can be applied in the rotation distance (Def.~\ref{def:rotationDistance}). In the case of spherical signals, we adopt the normalized spherical norm.
%
%
\begin{definition}[Normalized spherical norm] \label{def:rotationNorm}
	For a spherical signal $x: \mbS_{2} \to \reals$, the normalized spherical norm is defined as
    \begin{gather} \label{eq:rotationNorm}
    \| x \|^2 = \frac{1}{2 \pi^2}\int x(\theta_u,\phi_u)^2 d\theta_u d \phi_u.
    \end{gather}
    where $\theta_u \in [0,\pi]$ and $\phi_u \in [0,2\pi)$ describe the support of the spherical signal $x$ in the spherical coordinate system [cf. \eqref{eq:sphericalPoint}]. For multi-feature spherical signals $\bbx = \{x^{f}\}_{f=1}^{F}$, we define $\|\bbx\|^2 = \sum_{f=1}^{F} \| x^{f}\|^2$.
\end{definition}
%
%

\noindent This is a proper norm in the sense that it is (i) absolutely scalable, (ii) positive definite, and (iii) satisfying the triangular inequality. Equipping spherical signals with a norm allows us to define the following normed space
\begin{equation} \label{eq:L2spherical}
    \mbL^{2}(\mbS_{2}) = \big\{ \bbx: \mbS_{2} \to \reals^{F} \ : \ \| \bbx \| < \infty \big\}.
\end{equation}
The $\mbL^{2}(\mbS_{2})$ space holds for single-feature signals by setting $F=1$ in \eqref{eq:L2spherical} and using the corresponding definition of the normalized spherical norm [cf. Def.~\ref{def:rotationNorm}]. In essence, $\mbL^{2}(\mbS_{2})$ is the space of all finite-energy spherical signals.

Note that, if $\hhatx(u) = x(r (u)) = x_{r}(u)$ is a rotated version of the spherical signal $x$ by $r$, then $\| x - \hhatx\|_{\SO3} = \|x - x_{r}\|_{\SO3} = 0$. Under this dissimilarity measure given by the rotation distance (Def.~\ref{def:rotationDistance}), Prop.~\ref{prop:rotationEquivarianceConv} can be reframed with respect to the following corollary.

%
\begin{corollary} \label{cor:permutationEquivarianceConv}
Let $x, \hhatx$ be two spherical signals, and $H$ be a spherical convolutional filter. Then, it holds that
\begin{equation} \label{eq:permutationEquivarianceDistance}
\| x - \hhatx \|_{\SO3} = 0 \ \Rightarrow \ \| H(x) - H(\hhatx) \|_{\SO3} = 0.
\end{equation}
\end{corollary}
%
%
\begin{proof}
See \ref{proof:CorollaryEquivarianceCov}.
\end{proof}
%
%
\noindent Note that the rotation distance [cf. \eqref{eq:rotationDistance}] applies to the filter output since the latter is also a spherical signal.

%
\begin{remark}[Group distance]\label{rmk:groupNorm} \normalfont
    Definition \ref{def:rotationDistance} can be extended to any group. In particular, let $\mathsf{G} = \mathsf{G}(\mathbb{X}, \circ)$ be a group with elements $g \in \mathsf{G}$ defined over space $\mathbb{X}$, and $x, \hat{x}: \mathbb{X} \to \mathbb{R}$ be two group signals mapping $\mathbb{X}$ to the real line $\mathbb{R}$. The group distance with respect to $\mathsf{G}$ is defined as
    \begin{equation} \label{eq:groupDistance}
    \| x-\hhatx \|_{\mathsf{G}} = \inf_{g \in \mathsf{G}} \| x - \hhatx_{g} \|
    \end{equation}
    where $\hhatx_{g}(u) = \hhatx(g (u))$ is the signal operated by $g$ and $\| \cdot \|$ can be any valid norm. The rotation distance is a particular case of this definition for $\mathsf{G} = \SO3$.
\end{remark}
%


\subsection{Stability to diffeomorphism perturbations} \label{subsec:stabilityConv}

With above preliminaries in place, we formally characterize the stability of Spherical CNNs to general structure perturbations under the rotation distance (Def.~\ref{def:rotationDistance}). In other words, we shift focus on arbitrary structure perturbations that are different from but close to rotation operations. 



We consider the general structure perturbation $\tau: \mbS_{2} \to \mbS_{2}$ that displaces each point $u \in \mbS_2$ to another point $\tau(u) \in \mbS_2$. Such a perturbation can be represented as a set of local rotations, i.e., the rotation whose characterization changes depending on the specific point $u \in \mbS_2$ on which it is applied to. Let $r_{u} \in \SO3$ be the rotation operation that maps the point $u_{0} = (0, 0) \in \mbS_{2}$ to the point $u \in \mbS_{2}$ along the shortest arc, i.e., $r_u(u_0)=u$, and similarly $r_{\tau (u)}$ be the rotation operation that maps the point $u_0$ to the perturbed point $\tau(u)$ along the shortest arc, i.e., $r_{\tau (u)} (u_{0}) = \tau (u)$. Combining the notions of $r_{u}$ and $r_{\tau (u)}$, we can formally define the general structure perturbation as the diffeomorphism perturbation in the following definition.

%
\begin{figure}%
    \centering
    \begin{subfigure}{0.49\columnwidth}
        \centering
    \includegraphics[width=0.6\columnwidth]{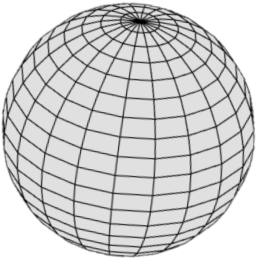}%
    \caption{Unperturbed space}%
    \label{subfig:unperturbedSpace}%
    \end{subfigure}
\hfill
    \begin{subfigure}{0.49\columnwidth}
        \centering
    \includegraphics[width=0.6\columnwidth]{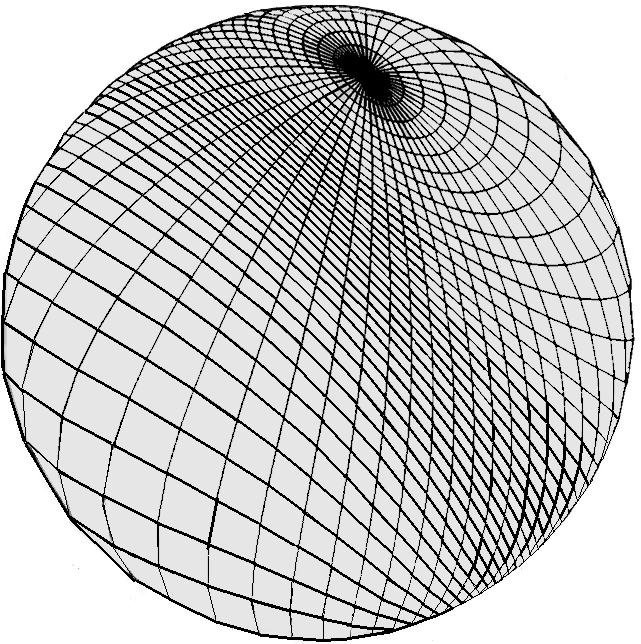}%
    \caption{Perturbed space}%
    \label{subfig:perturbedSpace}%
    \end{subfigure}%
    \caption{Diffeomorphism perturbation on the sphere [cf. Def~\ref{def:rotationDiffeomorphism}]. (a) Unperturbed sphere. (b) Perturbed sphere after a diffeomorphism perturbation.}
    \label{fig:diffeomorphism}
\end{figure}
%

%
\begin{definition}[Diffeomorphism perturbation] \label{def:rotationDiffeomorphism}
    Let $\tau : \mbS_{2} \to \mbS_{2}$ be a diffeomorphism, i.e., a bijective differentiable function whose inverse $\tau^{-1} : \mbS_{2} \to \mbS_{2}$ is differentiable as well. The diffeomorphism is equivalent to a set of local rotations $\{\tau_{u}\}_{u \in \mbS_2}$, where each local rotation $\tau_u \in \SO3$ is defined as
    \begin{equation} \label{eq:rotationDiffeomorphism}
        \tau_{u} = r_{\tau (u)} \circ r_{u}^{-1},~ \forall~ u \in \mbS_2,
    \end{equation}
that maps the point $u$ to the perturbed point $\tau(u)$, i.e., $\tau(u) = \tau_u(u)$ for all $u \in \mbS_2$. The diffeomorphism perturbation is defined as the set of local rotations in \eqref{eq:rotationDiffeomorphism} that satisfies $\|\tau\| < \infty$ and $\|\nabla \tau\| < \infty$ for
    \begin{equation} \label{eq:rotationDiffeomorphismDeviation}
    \|\tau \| := \max_{u \in \mbS_2} \Big\{ \beta^{\tau_{u}} \Big\},
    \end{equation}
    where $\beta^{\tau_{u}}$ is the rotation angle of the local rotation $\tau_{u}$ [cf. \eqref{eq:rotationDiffeomorphism}] \footnote{The rotation angle $\beta^{\tau_{u}}$ indicates the maximal geodesic distance that can be induced by the local rotation $\tau_u$ on the unit sphere.}; and for
    \begin{equation} \label{eq:rotationDiffeomorphismGradient}
    \| \nabla \tau \| = \max_{(\theta_u, \phi_u) \in [0,\pi] \times [0,2\pi)}\bigg\{ \Big| \frac{\partial \tau^{\theta}}{\partial \theta_u} \Big|, \Big|\frac{\partial \tau^{\phi}}{\partial \phi_u} \Big| \bigg\}.
    \end{equation}
    where $\tau^{\theta}$ and $\tau^{\phi}$ are polar and azimuthal angle displacements induced by $\tau$ such that $\tau^{\theta}(\theta_u, \phi_u) = \tau_u^\theta(\theta_u, \phi_u)$ and $\tau^{\phi}(\theta_u, \phi_u) = \tau_u^\phi(\theta_u, \phi_u)$ for all $u \in \mbS_2$. The signal resulting from a diffeomorphism perturbation to the spherical structure of a given signal $x$ is denoted by $x_{\tau}$ such that
    \begin{equation} \label{eq:perturbedSignal}
    x_{\tau}(u) = x\big(\tau (u)\big) = x\big(\tau_u (u)\big)\ ,~\forall~u \in \mbS_{2}.
    \end{equation}
\end{definition}

The diffeomorphism perturbation (Def.~\ref{def:rotationDiffeomorphism}) is, essentially, a set of local rotations where rotation axes and rotation angles depend on specific points on the sphere. We remark that since rotations can displace a point $u \in \mbS_2$ to any other point on the sphere, this representation can model any structure perturbation in the spherical space. Under the rotation distance [cf. \eqref{eq:rotationDistance}], we have
\begin{equation} \label{eq:perturbationModuleRotation}
    \| x - x_\tau \|_\SO3 = \inf_{r \in \SO3} \| x - x_{\tau \circ r} \|.
\end{equation}
Thus, if we want to measure the \emph{size of the preturbation}, we need to do so modulo rotations. Define the closest rotation $r^*$ to the diffeomorphism $\tau$ as
    \begin{equation} \label{eq:closestRotation}
    r^* = \argmin_{r \in \SO3} \max \{ \| \tau \circ r^{-1} \|, \| \nabla ( \tau \circ r^{-1} ) \| \}
    \end{equation}
where $\tau \circ r^{-1}$ is also a diffeomorphism but modulo the rotation $r$. The \emph{perturbation size} modulo rotations is then measured by [cf. \eqref{eq:rotationDiffeomorphismDeviation}, \eqref{eq:rotationDiffeomorphismGradient}]
    \begin{equation} \label{eq:rotationSize}
    \| \tau \|_\SO3 = \big\| \tau \circ (r^{*})^{-1} \big\|\ , \ \| \nabla \tau \|_\SO3 = \big\| \nabla \big(\tau \circ (r^{*})^{-1}\big) \big\|.
    \end{equation}
In this way, \eqref{eq:rotationSize} links the perturbation size to the rotation distance [cf. \eqref{eq:rotationDistance}] and determines how far from a rotation the actual perturbation $\tau$ is. Note that if the diffeomorphism perturbation is the rotation operation, all local rotations are same and $\| \tau \|_\SO3=\| \nabla \tau \|_\SO3\!=\!0$. Likewise, $\| \tau \|_\SO3=\| \nabla \tau \|_\SO3\!=\!0$ implies that the diffeomorphism perturbation is the rotation, i.e., $\tau \in \SO3$. An example of the diffeomorphism perturbation is illustrated in Fig.~\ref{fig:diffeomorphism}.

We have now defined the diffeomorphism perturbations of the spherical signal, as those that modify the spherical space $\mbS_{2}$ by using local rotations, i.e. rotations whose axes and angles vary as specific points in the spherical space change [cf. Def.~\ref{def:rotationDiffeomorphism}]. Next, we restrict our attention to the family of \emph{Lipschitz} filters.
%
\begin{definition}[Lipschitz filter] \label{def:Lipschitz}
    A spherical convolutional filter $h:\mbS_{2} \to \reals$ is \emph{Lipschitz}, if there exists a constant $C_h$ such that
    \begin{equation} \label{eq:Lipschitz}
        |h(u)| \leq C_{h} \ , \ \frac{|h(u_1) - h(u_2)|}{\text{arc} ( u_1, u_2 )} \leq C_{h}
    \end{equation}
    for all $u,u_1,u_2 \in \mbS_2$, where $\text{arc}(u_1, u_2)$ is the length of the shortest arc between $u_1$ and $u_2$ on the sphere. When we consider multi-feature signals, we say a filter $H$ [cf. \eqref{eq:filterBankConv}] is Lipschitz if each of the $FG$ single-feature filters $h^{fg}$ satisfies \eqref{eq:Lipschitz}.
\end{definition}
%

We can now proceed to characterize the stability of spherical convolutional filters to diffeomorphism perturbations [cf. Def.~\ref{def:rotationDiffeomorphism}] in the input spherical signal.

%
\begin{theorem}[Stability of spherical filters] \label{thm:stabilityConv}
Let $x \in\mathbb{L}^2(\mbS_2) $ be a spherical signal, and $H$ be a Lipschitz filter with constant $C_h$ [cf. Def~\ref{def:Lipschitz}]. Consider the diffeomorphism perturbation $\tau$ [cf. Def~\ref{def:rotationDiffeomorphism}] that satisfies
\begin{equation} \label{eq:perturbationSizeCondition}
    \| \tau \|_\SO3 \le \epsilon \ , \ \| \nabla \tau \|_\SO3 \le \epsilon \le \frac{1}{2}.
\end{equation}
Then, for the perturbed input spherical signal $x_{\tau}$ [cf. \eqref{eq:perturbedSignal}], it holds that
\begin{equation}\label{eq:stabilityConv}
    \| H(x) - H(x_{\tau})\|_{\SO3} \leq 8 C_h \epsilon \| x \| + \ccalO(\epsilon^{2})
\end{equation}
\end{theorem}
%
%
\begin{proof}
See \ref{proof:theoremStabilityCov}.
\end{proof}
%
%
Theorem \ref{thm:stabilityConv} establishes that spherical convolution filters are Lipschitz stable to diffeomorphism perturbations (modulo rotations). More specifically, the output difference between filtering $x$ and a perturbed version $x_{\tau}$ of it, is bounded linearly by the perturbation size $\epsilon$ with respect to the stability constant $8 C_{h}$. When the diffeomorphism perturbation is close to the rotation operation (i.e., $\epsilon \to 0$), the bound approaches zero and the spherical convolutional filter maintains performance. The stability constant implies the role of the filter property. A careful design can reduce the Lipschitz constant $C_{h}$ and thus, lead to more stable filters. The constant term $8$ depends on the domain, which can be improved by further restricting the class of filters. Note that Theorem \ref{thm:stabilityConv} reduces to Corollary~\ref{cor:permutationEquivarianceConv} if the diffeomorphism perturbation $\tau$ reduces to the rotation operation $r$, i.e. $\| \tau \|_\SO3=\| \nabla \tau \|_\SO3=0$.

In summary, Theorem~\ref{thm:stabilityConv} demonstrates that spherical convolutional filters yield similar (stable) outputs when signals are close to being rotated versions of each other. The result provides the stability analysis for arbitrary structure perturbations in spherical signals, where the rotation equivariance is a special case when particularizing structure perturbations as rotation operations. In other words, we extend the rotation equivariance to general stability property.

%% file: 04-stabilitySCNN.tex


Spherical CNNs $\Phi(x;\ccalH)$ [cf. \eqref{eq:singleFeatureSCNN}] are nonlinear processing architectures consisting of spherical convolutional filters and pointwise nonlinearities, which have the potential to learn nonlinear representations from spherical signals. Based on the stability result of spherical convolutional filters, we characterize the effect of structure perturbations on the output of Spherical CNNs, to illustrate how they exploit the rotational structure inherent in spherical signals.



\subsection{Rotation equivariance of Spherical CNNs} \label{subsec:rotationEquivarianceSCNN}

The fact that spherical convolutional filters are rotation equivariant (Prop.~\ref{prop:rotationEquivarianceConv}) extends immediately to Spherical CNNs. This is because of the pointwise nature of the nonlinearity, which does not affect the data structure.

%
\begin{proposition} \label{PropositionEquivarianceSCNN}
    Let $x$ be a spherical signal, $\Phi(\cdot;\ccalH)$ be a Spherical CNN [cf. \eqref{eq:singleFeatureSCNN}], and $r \in \SO3$ be a rotation. Denote by $x_{r}(u) = x(r (u))~\forall~u \in \mbS_2$ the rotated version of the input spherical signal $x$ by $r$. Then, it holds that
	\begin{equation} \label{eq:PropositionEquivarianceSCNN}
		\| \Phi(x;\ccalH) - \Phi(x_{r};\ccalH) \|_{\SO3} = 0.
	\end{equation}
\end{proposition}
%
%
\begin{proof}
See \ref{proof:propositionEquivarianceSCNN}.
\end{proof}
%

Proposition \ref{PropositionEquivarianceSCNN} demonstrates that the output of Spherical CNN applied on the rotated spherical signal $y= \Phi(x_{r};\ccalH)$ is the corresponding rotated output of Spherical CNN applied on the unperturbed spherical signal $y_{r}(u) \!=\! y(r(u))$ for all $u \in \mbS_2$. This implies that Spherical CNNs capture the same information, irrespective of the specific rotated version of spherical signal. Proposition \ref{PropositionEquivarianceSCNN} suggests that Spherical CNNs generalize well since they can learn how to process all rotated versions of a given signal, by learning how to process one of them.


\subsection{Stability to diffeomorphism perturbations of Spherical CNNs} \label{subsec:stabilitySCNN}

We proceed to consider general structure perturbations instead of fundamental rotation operations. The stability of Spherical CNNs to diffeomorphism perturbations (Def.~\ref{def:rotationDiffeomorphism}) is inherited from that of spherical convolutional filters. The inclusion of nonlinearity in the architecture affects the stability constant. In particular, we consider the Lipschitz nonlinearity as defined next.

%
\begin{definition}[Lipschitz nonlinearity] \label{def:lipschitzNonlinearity}
A nonlinearity $\sigma:\reals  \to \reals$ satisfying $\sigma(0)=0$ is Lipschitz if there exists a constant $C_\sigma >0$ such that
\begin{gather}\label{eq:lipschitzNonlinearity}
| \sigma(a) - \sigma(b) | \le C_\sigma | a - b |, ~ \forall~ a, b \in \mathbb{R}.
\end{gather}
\end{definition}
%

When Spherical CNNs are built from Lipschitz filters (Def.~\ref{def:Lipschitz}) and use Lipschitz nonlinearities (Def.~\ref{def:lipschitzNonlinearity}), they inherit the stability to diffeomorphism perturbations (Def.~\ref{def:rotationDiffeomorphism}) from spherical convolutional filters [cf. Theorem~\ref{thm:stabilityConv}].

%
\begin{theorem}[Stability of Spherical CNNs] \label{thm:stabilitySCNN}
    Let $x \in\mathbb{L}^2(\mbS_2) $ be a spherical signal, and $\Phi(\cdot;\ccalH)$ be a Spherical CNN [cf. \eqref{eq:singleFeatureSCNN}] consisting of $L$ layers, with $F_{\ell} \!=\! F$ features per layer, built with Lipschitz filters with constant $C_h$ [cf. Def~\ref{def:Lipschitz}] and Lipschitz nonlinearities with constant $C_\sigma$ [cf. Def.~\ref{def:lipschitzNonlinearity}]. Consider a diffeomorphism perturbation $\tau$ [cf. Def~\ref{def:rotationDiffeomorphism}] that satisfies
    \begin{equation} \label{eq:perturbationSizeConditionSCNN}
    \| \tau \|_\SO3 \le \epsilon \ , \ \| \nabla \tau \|_\SO3 \le \epsilon \le \frac{1}{2}.
    \end{equation}
    Then, for the perturbed input spherical signal $x_{\tau}$ [cf. \eqref{eq:perturbedSignal}], it holds that
    \begin{equation}\label{eq:stabilitySCNN}
    \| \Phi(x;\ccalH) - \Phi(x_{\tau};\ccalH)\|_{\SO3} \leq 8 C_{h}^{L} C_{\sigma}^{L} F^{L-1}\epsilon \| x \|+ \ccalO(\epsilon^{2}).
    \end{equation}
\end{theorem}
%
%
\begin{proof}
See \ref{proof:stabilitySCNN}.
\end{proof}
%
%

Theorem \ref{thm:stabilitySCNN} determines that Spherical CNNs are Lipschitz stable to structure perturbations. More specifically, the output difference of Spherical CNN induced by diffeomorphism perturbations (modulo rotations) is upper bounded by a term that depends proportionally on the size of the perturbation $\epsilon$. When the perturbation size $\epsilon \to 0$, diffeomorphism perturbations approach rotation operations, the bound reduces to zero, and Spherical CNNs maintain performance. In other words, if two signals are close to being rotated versions of each other, their outputs of the Spherical CNN will be close as well. 

While the stability bound may have a similar form as that of spherical convolutional filters [Thm.~\ref{thm:stabilityConv}], the stability constant accounts for the effects of different architecture components. In particular, it is the product of three terms:

\smallskip
\begin{enumerate}[(1)]

\item The first term $8C_h^L$ captures the property of spherical convolutional filters. A larger Lipschitz constant $C_h$ [cf. \eqref{eq:Lipschitz}] allows larger input-output expansivity through filtering propagation and more variability in filter values between nearby spherical points, leading to less stability to structure perturbations that displace spherical points arbitrarily. Reducing $C_h$ yields spherical convolutional filters that are less expansive and change more slowly, resulting in the improved stability. However, this improvement comes at expenses of the information loss and the expressive power. The former is because less expansive filters propagate less information from the input to the output, and the latter is because filters become more flat in the spherical surface reducing the representative power. 

\item The term $C_\sigma^L$ captures the role of the nonlinearity. In particular, $C_\sigma$ is typically one indicating the non-expansivity of the nonlinearity, such as the absolute value, the ReLU, the Tanh, etc.

\item The last term $F^{L-1}$ represents the impact of the Spherical CNN architecture, namely, the number of layers ($L$) and the number of spherical convolutional filters per layer ($F$). A wider architecture with more features and a deeper architecture with more layers yields a looser bound and degrades the stability. This can be explained by the fact that more filters are involved in the Spherical CNN, which amplifies the effect of structure perturbations passing through the architecture.

\end{enumerate}

Overall, Spherical CNNs inherit the stability to structure perturbations from spherical convolutional filters. This explains how Spherical CNNs exploit the underlying structure present in $3$D data and how they maintain performance under arbitrary perturbations close to rotations.

\begin{remark}\label{rmk:groupNote}\normalfont
The results can be extended to any generic group $\mathsf{G}$ besides the rotation group $\SO3$. I.e., the stability to general diffeomorphism perturbations applies to the group convolutional filter [cf. \eqref{eq:groupConv}]; hence, applies to the group convolutional neural network. Proofs of propositions and theorems can be carried out with similar processes, while the key step may be to find a suitable mathematical representation for the data structure of interest to perform theoretical analysis. In the rotation group $\SO3$, for instance, the spherical coordinate system is utilized to describe spherical signals and the Euler parameterization is used to characterize the rotation operation. 

\end{remark}

%% file: 05-experimentsStabilitySO3.tex


%
\begin{figure}[t]
    \centering
    \includegraphics[height=0.12\textheight, keepaspectratio]{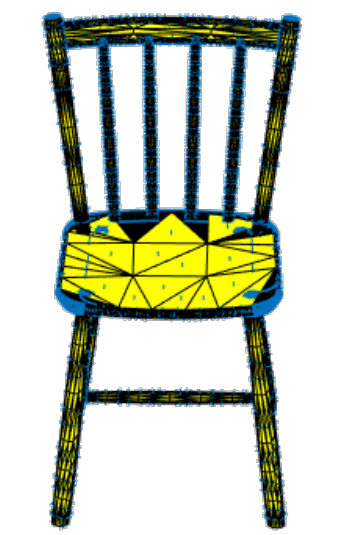}
    \qquad \qquad \quad
    \includegraphics[height=0.12\textheight, keepaspectratio]{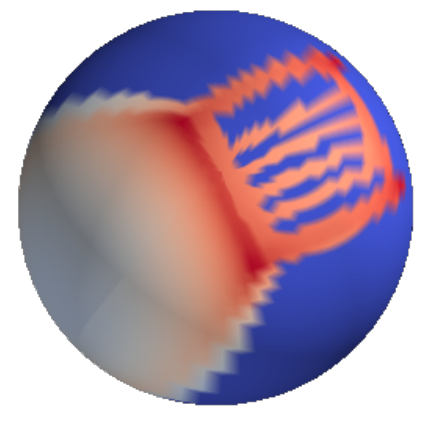}
    \caption{3D model (left) and spherical signal (right) of a chair. }
    \label{3Dchair}
\end{figure}
%

\begin{figure*}%
\centering
\begin{subfigure}{0.3\columnwidth}
    \centering
\includegraphics[height=0.12\textheight]{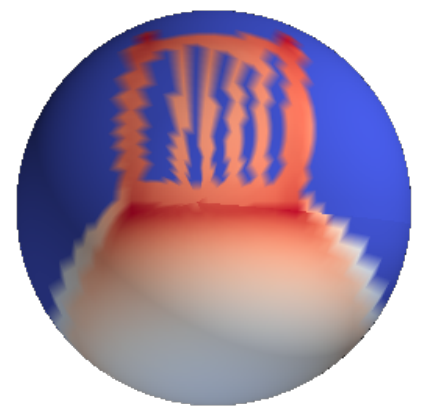}%
\centering
\caption{}%
\label{subfiga_rotation}%
\end{subfigure}\hfill
\begin{subfigure}{0.3\columnwidth}
    \centering
\includegraphics[height=0.12\textheight]{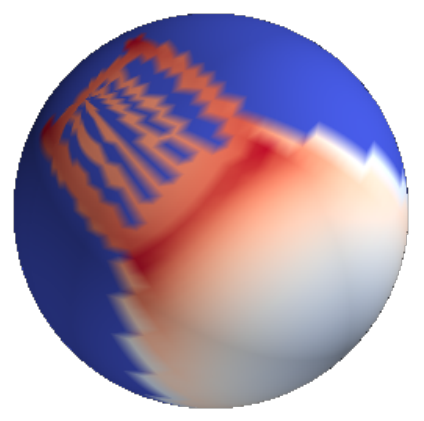}%
\caption{}%
\label{subfigb_rotation}%
\end{subfigure}\hfill%
\begin{subfigure}{0.3\columnwidth}
    \centering
\includegraphics[height=0.12\textheight]{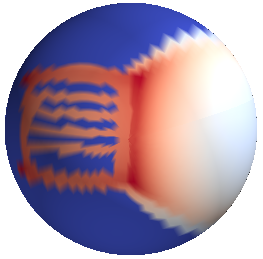}%
\caption{}%
\label{subfigc_rotation}%
\end{subfigure}%
\caption{Rotation operations on the spherical signal. \subref{subfiga_rotation} $45$ degree rotation. \subref{subfigb_rotation} $90$ degree rotation. \subref{subfigc_rotation} $135$ degree rotation. }\label{3Dchairrotation}\vspace{-2mm}
\end{figure*}

We have proved that Spherical CNNs are stable to general structure perturbations (i.e., diffeomorphism perturbations). This illustrates how Spherical CNNs exploit the data structure to improve representative power and learning capacity. As a matter of fact, they have already been shown successful in classification tasks \cite{esteves2018learning}. Thus, in what follows, we focus solely on corroborating the stability property by numerical experiments with admissible perturbations; for performance comparison with other learning methods, please refer to \cite{esteves2018learning}.

\myparagraph{Problem setting and dataset.} The shape classification problem of 3D object is considered on ModelNet40 dataset \cite{wu20153d}, i.e., given a spherical signal, the goal is to find out which class its represented object belongs to. There are 40 classes in the dataset, where we use $9683$ samples for training and $29595$ samples for testing. We parametrize the spherical signal in a $64 \times 64$ resolution---see Fig. \ref{3Dchair} for the 3D model of a chair and the associated spherical signal.

\begin{table}
\begin{center}
\caption{Test classification accuracy of the Spherical CNN for the original data, $45$ degree rotated data, $90$ degree rotated data, and $135$ degree rotated data. Root mean square error (RMSE) of the Spherical CNN output features for the $45$ degree rotated data, $90$ degree rotated data and $135$ degree rotated data.}
\label{table:rotationEquivariance}
\begin{tabular}{|c|c|c|}
\hline
Rotation angle & Classification accuracy & RMSE \\ \hline
\specialrule{0em}{0.1pt}{0.1pt}
$0$ degree & $0.864$ & 0 \\ \hline
\specialrule{0em}{0.6pt}{0.6pt}
$45$ degree & $0.864$ & $1.67 \cdot 10^{-3}$ \\  \hline
$90$ degree & $0.864$ & $1.58 \cdot 10^{-3}$ \\ \hline
$135$ degree & $0.864$ & $2.67\cdot 10^{-3}$ \\
\hline
\end{tabular}
\end{center}  \vspace{-4mm}
\end{table}

\myparagraph{Architectures and training.} We consider the Spherical CNN of $8$ layers, each containing $16$, $16$, $32$, $32$, $64$, $64$, $128$ and $128$ spherical convolutional filters and the ReLU nonlinearity. At the readout layer, we apply a global weighted average pooling for a descriptor vector and the latter is projected into the number of object classes. We train the architecture for $50$ epochs with the ADAM optimizer and a batch size of $16$ samples. The learning rate is $1\cdot 10^{-3}$, which is divided by $5$ on epochs $30$ and $40$ respectively. We start by 
considering structure perturbations as fundamental rotation operations, and proceed to arbitrary diffeomorphism perturbations.

\myparagraph{Rotation equivariance.} Given a trained Spherical CNN, we consider the testing data perturbed by $45$ degree, $90$ degree and $135$ degree rotation operations, respectively. Fig. \ref{3Dchairrotation} displays the resulting spherical signals under three rotations, and Table \ref{table:rotationEquivariance} shows the classification accuracies over the original and rotated testing data. We see that the classification accuracies remain the same, evidencing that the Spherical CNN is rotation equivariant. We also show the root mean squared error (RMSE) between output features of the Spherical CNN when applied to the original signal and the rotated signal. The results indicate that they are virtually the same output features, as expected in theory. 

\begin{figure*}%
\centering
\begin{subfigure}{0.25\columnwidth}
    \centering
\includegraphics[width=0.71\linewidth, height = 0.65\linewidth]{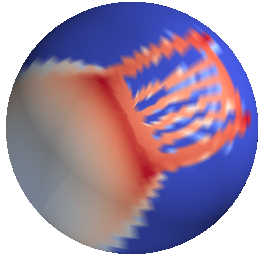}%
\caption{}%
\label{subfiga_diffeomorphism}%
\end{subfigure}\hfill%
\begin{subfigure}{0.25\columnwidth}
    \centering
\includegraphics[height = 0.65\linewidth]{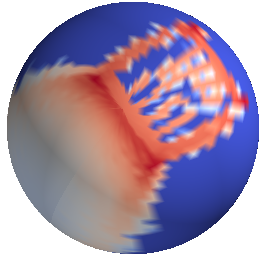}%
\caption{}%
\label{subfigb_diffeomorphism}%
\end{subfigure}\hfill
\begin{subfigure}{0.25\columnwidth}
    \centering
\includegraphics[height = 0.65\linewidth]{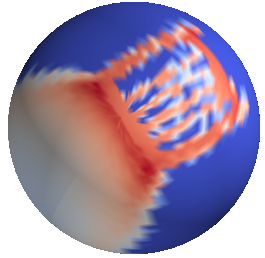}%
\caption{}%
\label{subfigc_diffeomorphism}%
\end{subfigure}\hfill
\begin{subfigure}{0.25\columnwidth}
    \centering
\includegraphics[height = 0.65\linewidth]{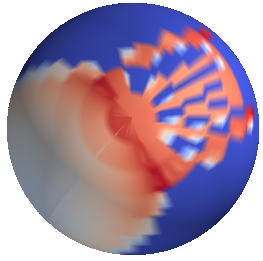}%
\caption{}%
\label{subfigd_diffeomorphism}%
\end{subfigure}%
\caption{Diffeomorphism perturbations on the spherical signal. \subref{subfiga_diffeomorphism} Type 1. \subref{subfigb_diffeomorphism} Tyep 2. \subref{subfigc_diffeomorphism} Type 3. \subref{subfigd_diffeomorphism} Type 4. }\label{3Dchairdiffeomorphism}\vspace{-2mm}
\end{figure*}

\myparagraph{Diffeomorphism perturbations.} We now test the output change of the Spherical CNN when the input signal is subject to four types of diffeomoprhism perturbations. Namely, we carry out different local rotations along the latitude, indicating different perturbation severity. For type $1$, we rotate every other sampled point at each latitude with a random degree drawn form $[-3, 3]$, where we assume the clockwise direction as the positive direction. Type $2$ rotates every other sampled point with a random degree drawn from $[-6, 6]$. Type $3$ considers the local rotation at each sampled point with a random degree drawn from $[-3, 3]$. Finally, type $4$ perturbs blocks of $3$ samples at each latitude separately, rotating the second point to the third point and interpolating the values of the remaining sampled point, where the maximal degree change is approximately $6$ degrees. Fig. \ref{3Dchairdiffeomorphism} displays perturbed spherical signals that are associated to these four types of diffeomorphism perturbations.

\begin{table}
\begin{center}
\caption{Test classification accuracy of the Spherical CNN for diffeomorphism perturbation types 1-4. Root mean square error (RMSE) of the Spherical CNN output features for diffeomorphism perturbation types 1-4.}
\label{table3}
\begin{tabular}{|c|c|c|}
\hline
Diffeomorphism perturbation & Classification accuracy & RMSE \\ \hline
Type $1$ & $0.863$ & $0.0668$ \\ \hline
Type $2$ & $0.856$ & $0.0872$ \\ \hline
Type $3$ & $0.858$ & $0.0778$ \\ \hline
Type $4$ & $0.828$ & $0.1254$ \\
\hline
\end{tabular}
\end{center}  \vspace{-4mm}
\end{table}

\myparagraph{Diffeomorphism stability.} Table \ref{table3} shows the classification accuracies of the trained Spherical CNN when assuming that testing spherical signals are perturbed by aforementioned four types of diffeomorphism perturbations. In general, the Spherical CNN exhibits strong robustness in four cases, as expected from Theorem~\ref{thm:stabilitySCNN}. We see that type $1$ has little effect on the classification accuracy, while types $2$ and $3$ slightly decrease the accuracy due to the increase of the maximal degree change and the increased number of perturbed sampled points. Type $4$ is most severe as observed from Fig. \ref{subfigd_diffeomorphism}, while the Spherical CNN only suffers from $0.036$ accuracy loss, emphasizing its stability to diffeomorphism perturbations. Table \ref{table3} also shows the RMSE of the output features at final layer under above diffeomorphism perturbations. The RMSE maintains low values in all cases, which validates the stability of the Spherical CNN to diffeomorphism perturbations. 

%% file: 06-conclusionsStabilitySO3.tex


This paper investigated the stability of Spherical CNNs to general structure perturbations in spherical signals. We first considered spherical signals perturbed by fundamental rotation operations, and proved explicitly the output of the Spherical CNN applied on the rotated signal is the rotated output of the Spherical CNN applied on the unperturbed signal. This property, referred to as the rotation equivariance, implies the same information is obtained irrespective of rotation operations and thus, we defined the rotation distance for stability analysis modulo rotations. We then considered spherical signals perturbed by arbitrary structure perturbations, and modelled the latter as diffeomorphism perturbations. We showed the output difference of Spherical CNN induced by the diffeomorphism perturbation is upper bounded proportionally by the perturbation size under the rotation distance. This result indicates that Spherical CNNs maintain performance when structure perturbations are close to rotation operations. These theoretical findings also show the explicit role of the filter property, nonlinearity, architecture width and depth on the stability of Spherical CNNs, which identifies handle to improve the robustness. The theory was corroborated through numerical experiments of 3D object classification.

%% file: A1-proofsStabilitySO3.tex



\section{Proof of Proposition~\ref{prop:rotationEquivarianceConv}} \label{proof:PropositionEquivarianceCov}

\begin{proof}
Given the rotated spherical signal $x_r$, the output of the spherical convolutional filter $H$ is given by
\begin{align} \label{eq:prprop1}
H(x_r)(u)&=\int_{\SO3} h\big(\hat{r}^{-1} (u) \big)x\Big(\!r \big(\hat{r}(u_0)\big)\!\Big)d\hat{r}\\
&= \int_{\SO3} \! h\Big(\! \big(r \circ \hat{r} \big)^{-1} \big(r (u) \big)\!\Big) x\Big(\!(r \circ \hat{r})(u_0)\!\Big)d\hat{r},\nonumber
\end{align}
where $u_0=(0,0)\in \mbS_2$ is given by $\theta_{u_0} = \phi_{u_0} = 0$ and $r(u)\in \mbS_2$ is the rotated point of $u$ by $r$. By defining $r' = r \circ \hat{r}$, we can rewrite \eqref{eq:prprop1} as
\begin{equation} \label{eq:prprop2}
\begin{split}
H(x_r)(u)&=\int_{\SO3} h\Big({r'}^{-1} \big( r (u) \big)\Big)x\big(r'(u_0)\big)dr'\\
&=h\ast_\SO3 x \big(r (u) \big) =H(x)(r (u) )
\end{split}
\end{equation}
Note that \eqref{eq:prprop2} holds for any point $u \in \SO3$, and thus it holds for the spherical signal $x$. By letting $y = H(x)$, we complete the proof thta $H(x_r) = y_r$.
\end{proof}


\section{Proof of Corollary~\ref{cor:permutationEquivarianceConv}} \label{proof:CorollaryEquivarianceCov}

\begin{proof}
Given $\| x - \hhatx \|_{\SO3} = 0$ under the rotation distance [Def. \ref{def:rotationNorm}], there exists a rotation $\hat{r}$ such that $\hhatx = x_{\hat{r}}$ is a rotated version of signal $x$. Denote by $y \!=\! H(x)$ the output of spherical convolutional filter and we have $H(x_{\hat{r}}) = y_{\hat{r}}$ resulting from Proposition \ref{prop:rotationEquivarianceConv}. Under the rotation distance, we get
\begin{align}
\| H \left( x \right) - H(\hhatx) \|_{\SO3} &= \inf_{r\in \SO3} \| y - y_{\hat{r} \circ r} \| \le\! \| y - y_{{\hat{r}} \circ {\hat{r}}^{-1} } \| \!=\! 0.
\end{align}
Since $\| H \left( x \right) - H(\hhatx) \|_{\SO3} \ge 0$ by the norm definition, we can then obtain $\| H ( x ) - H(\hhatx) \|_{\SO3} = 0$ completing the proof.
\end{proof}


\section{Proof of Theorem~\ref{thm:stabilityConv}} \label{proof:theoremStabilityCov}

\begin{proof}

Under the rotation distance, the output difference of the spherical convolutional filter induced by the diffeomorphism perturbation $\tau$ is
\begin{align}\label{eq:proofthm11}
\| H(x_\tau) &\!-\! H(x) \|_{\SO3} = \inf_{r \in \SO3} \|H( x_\tau ) \!-\! \left(H(x)\right)_r \| \\
& = \inf_{r \in \SO3} \|H( x_\tau ) \!-\! H( x_r ) \|  \le \|H( x_\tau ) \!-\! H(x_{r^*}) \| \nonumber
\end{align}
where $\left(H(x)\right)_r$ is the rotated version of the filter output $H(x)$ by $r$, and where the rotation equivariance of the spherical convolutional filter is used and $r^* \!\in \SO3$ is the closest rotation to $\tau$ defined as $r^* = \argmin_{r \in \SO3} \max \{ \| \tau \!\circ r^{-1} \|, \| \nabla (\tau \circ r^{-1}) \|\}$ [cf. \eqref{eq:closestRotation}]. Since the normalized spherical norm is rotation invariant, i.e., $\| x_r \| = \| x \|$ for all $r \in \SO3$, we have
\begin{align}\label{eq:proofthm12}
&\| H(x_\tau) \!-\! H(x_{r^*}) \| = \|\big(H( x_\tau )\big)_{({r^*})^{-1}} \!-\! \big(H(x_{r^*})\big)_{({r^*})^{-1}} \| \\
&= \| H( x_{\tau \circ ({r^*})^{-1} } ) \!-\! H(x_{ r^* \circ ({r^*})^{-1} }) \| = \| H( x_{\tau \circ ({r^*})^{-1}} ) - H( x) \|\nonumber
\end{align}
where $\big(H(x)\big)_{({r^*})^{-1}}$ is the rotated version of the filter output $H( x )$ by $({r^*})^{-1}$. Let $\tau^* = \tau \circ ({r^*})^{-1}$ be the diffeomorphism module the rotation $r^*$. Let also $Hx = H(x)$ and $P_{\tau}^*x = x_{\tau^*}$ be operator notations. By substituting \eqref{eq:proofthm12} into \eqref{eq:proofthm11}, we get
\begin{equation}\label{eq:proofthm13}
\| H( x_{\tau} ) - H( x) \|_\SO3 \le \| H P^*_\tau x - Hx \|  = \| ( H P^*_\tau - H )x \|.
\end{equation}
We consider the operator $HP^*_\tau - H$ and factorize it as
\begin{gather}\label{eq:proofthm14}
HP^*_\tau - H = \left( H - H {P^*_\tau}^{-1} \right)P^*_\tau.
\end{gather}
where ${P^*_\tau}^{-1}$ is the inverse operator of $P^*_\tau$, which exists since $\tau^*$ is a diffeomorphism. The triangle inequality allows
\begin{equation}\label{eq:prthm15}
\| HP^*_\tau\! -\! H \| \! =\! \|\big(\! H \!-\! H {P^*_\tau}^{-1} \!\big)P^*_\tau \|\!\le\! \|P^*_\tau \|\| H - H {P^*_\tau}^{-1} \|.
\end{equation}
Let us analyze two terms in the bound of \eqref{eq:prthm15} separately. 

For the first term $\| P^*_\tau \|$, a known procedure to bound the norm of an operator $P^*_\tau$ is to find an upper bound of $\| P^*_\tau x \|$ as $\| P^*_\tau x \|\le A \|x \|$ for any spherical signal $x$. Then, $A$ is the upper bound of $\| P^*_\tau \|$. Following this intuition, we have
\begin{align}\label{eq:prthm16}
&\| P^*_\tau x \|^2 \!=\! \frac{1}{2\pi^2} \int P^*_\tau x \left(\theta_u,\phi_u\right)^2 d\theta_u d\phi_u \\
& \!=\! \frac{1}{2\pi^2}\int\! x\! \left(\theta_u \!+\! \tau^{*\theta}(\phi_u,\theta_u), \phi_u \!+\!\tau^{*\phi}(\phi_u,\theta_u) \right)^2 d\theta_u d\phi_u \nonumber
\end{align}
where $\tau^{*\theta}$ and $\tau^{*\phi}$ are polar and azimuth angle displacements induced by $\tau^*$. Denote by $\hat{\theta}_u = \theta_u + \tau^{*\theta}(\phi_u,\theta_u)$ and $\hat{\phi}_u = \phi_u +\tau^{*\phi}(\phi_u,\theta_u)$ the variable substitutions, and observe that
\begin{align}\label{eq:prthm17}
d\hat{\theta}_u d\hat{\phi}_u &\!=\! \left( 1 \!+\! \frac{\partial \tau^{*\theta}}{\partial \theta_u}(\theta_u, \phi_u) \right)\!\left( 1 \!+\! \frac{\partial \tau^{*\phi}}{\partial \phi_u}(\theta_u, \phi_u) \right)d\theta_u d\phi_u \nonumber \\
& \ge \big(1- \| \nabla \tau \|_\SO3 \big)^2 d\theta_u d\phi_u,
\end{align}
where the last inequality is due to the definition of $\| \nabla \tau \|_\SO3$ [cf. \eqref{eq:rotationSize}]. Substituting \eqref{eq:prthm17} into \eqref{eq:prthm16} yields
\begin{align}\label{eq:prthm18}
\| P^*_\tau x \|^2 \!&\le \frac{1}{2\pi^2}\! \int x (\hat{\theta}_u, \hat{\phi}_u)^2 (1- \| \nabla \tau \|_\SO3)^{-2} d\hat{\theta}_u d\hat{\phi}_u \nonumber \\
&= (1- \| \nabla \tau \|_\SO3)^{-2} \| x \|^2 \le 4 \| x \|^2,
\end{align}
where $\| \nabla \tau \|_\SO3 \le 1/2$ is used in the last inequality. Thus, we have $\| P^*_\tau \| \!\le\! 2$. 

For the second term $\| H \!-\! H {P^*_\tau}^{-1} \|$, let ${P^*_\tau}^{-1}x(u)=x\left( \xi (u) \right)$ where $\xi$ is the inverse operation of $\tau^*$, and we have
\begin{align} \label{eq:prthm19}
&H x(u) = \int_{\SO3}\!h (r^{-1}(u) )x\big(r(u_0)\big)dr,\\
\label{eq:prthm110}&H {P^*_\tau}^{-1} x (u) \!=\!\! \int_{\SO3}\!\!\!\!\!\!\! h\big(r^{-1}(u) \big)x\Big(\xi \big(r (u_0) \big) \Big) dr. 
\end{align}
where $u_0 = (0,0)$ is the point given by $\theta_u = \phi_u =0$. For \eqref{eq:prthm19}, we can rewrite it by parametrizing rotations with the Euler angles [cf. \eqref{eq:rotationEuler} and \eqref{eq:sphericalConvRectangle}] as
\begin{align} \label{eq:prthm111}
H x (u) &=\frac{1}{8\pi^2}\! \int\! h \left(r^{-1}_{\phi_r\theta_r\rho_r} (u)\right) x(\theta_r, \phi_r) \sin(\theta_r) d\theta_r d\phi_r d\rho_r \nonumber\\
& = \frac{1}{8\pi^2}\! \int\! h \Big(  (r_{u_r} \circ r^z_{\rho_r})^{-1} (u)\Big) x(\theta_r, \phi_r) \sin(\theta_r) d\theta_r d\phi_r d\rho_r
\end{align}
where $r_{\phi_r\theta_r\rho_r}$ is the rotation characterized by the Euler angles $(\phi_r,\theta_r,\rho_r)$ and $r^{-1}_{\phi_r\theta_r\rho_r}$ is its inverse rotation, and where $r_{u_r}$ is the rotation that rotates the point $u_0=(0,0)$ to the point $u_r=(\theta_r, \phi_r)$ along the shortest arc, i.e., $r_{u_r}(u_0) = u_r$, $r^z_{\rho_r}$ is the rotation along $z$-axis by $\rho_r$ angles, and $(r_{u_r} \circ r^z_{\rho_r})^{-1} = {r^z_{\rho_r}}^{-1} \circ r_{u_r}^{-1}$ is its inverse rotation. For \eqref{eq:prthm110}, by letting $\xi_{r} = \xi \circ r$, we have
\begin{align} \label{eq:prthm112}
&H {P^*_\tau}^{-1} x (u) = \int_{\SO3} h \big( r^{-\!1} (u) \big)x\big(\xi_{r} (u_0)\big) dr.
\end{align}
Let $\xi_{r} (u_0) = u_r = (\theta_r, \phi_r)$ be the polar angle displacement and the azimuth angle displacement of $\xi_{r}$ at $u_0$. Since $\xi^{-1} = \tau^*$, we get
\begin{align} \label{eq:prthm113}
r (u_0) \!=\! \xi^{-1}\big(\xi_{r} (u_0)\big) \!=\! \tau^* (u_r) \!=\! \big(\theta_r \!+\! \tau^{*\theta}(\theta_r,\phi_r), \phi_r \!+\! \tau^{*\phi}(\theta_r,\phi_r)\big)
\end{align}
such that the normalized Haar measure is given by
\begin{equation}\label{eqprthm1Eulerangle}
    dr = \frac{1}{8\pi^2}\alpha(\theta_r, \phi_r) \sin\!\big(\theta_r \!+\! \tau^{*\theta}(\theta_r,\phi_r)\big)d\phi_{r}d\theta_{r}d \rho_{r},
\end{equation}
with $\alpha(\theta_r, \phi_r)=\left(\!1\!+\!\frac{\partial \tau^{*\theta}(\theta_r, \phi_r)}{\partial \theta_r}\! \right) \left(\!1\!+\!\frac{\partial \tau^{*\phi}(\theta_r, \phi_r)}{\partial \phi_r} \!\right)$. By parametrizing rotations with the Euler angles [cf. \eqref{eq:rotationEuler} and \eqref{eq:sphericalConvRectangle}] and using the results in \eqref{eq:prthm113} and \eqref{eqprthm1Eulerangle}, we can similarly rewrite \eqref{eq:prthm112} as
\begin{align} \label{eq:prthm114}
H {P^*_\tau}^{-1} x (u) =& \frac{1}{8\pi^2}\!\int\! x(\theta_r, \phi_r)\alpha(\theta_r, \phi_r)\sin\!\big(\theta_r \!+\! \tau^{*\theta}(\theta_r,\phi_r)\big)\cdot\nonumber \\
& h \Big(\! (r_{\tau^* (u_r)} \circ r^z_{\rho_r})^{-1} (u) \!\Big) d\theta_r d\phi_r d\rho_r 
\end{align}
where $r_{\tau^* (u_r)}$ is the rotation that rotates the point $u_0=(0,0)$ to the point $\tau^* (u_r\!)\!=\!\big(\theta_r \!+\! \tau^{*\theta}(\theta_r,\phi_r), \phi_r \!+\! \tau^{*\phi}(\theta_r,\phi_r)\big)$ along the shortest arc, i.e., $r_{\tau^* (u_r)} (u_0) \!=\! \tau^*(u_r)$, $r^z_{\rho_r}$ is the rotation along $z$-axis by $\rho_r$ angles, and $(r_{\tau^* (u_r)} \circ r^z_{\rho_r})^{-1} \!\!=\! {r^z_{\rho_r}}^{-1} \circ r_{\tau^* (u_r)}^{-1}$ is its inverse rotation.

By substituting \eqref{eq:prthm111} and \eqref{eq:prthm114} into the operator $Hx - H {P_\tau^{*}}^{-1}x$, we get
\begin{align} \label{eq:prthm116}
Hx(u)- H {P^*_\tau}^{-1} x (u) &\!=\!\frac{1}{8\pi^2}\!\int\! x(\theta_r, \phi_r)\! \left[ h\Big({r^z_{\rho_r}}^{-1} \big( r_{u_r}^{-1} (u)\big) \Big)\sin(\theta_r) \right.\nonumber\\
&\left.\!\!\!\!- h\Big(\! {r^{z}_{\rho_r}}^{-\!1}\big( r_{\tau^* (u_r)}^{-1}(u)\big) \!\Big)\alpha(\theta_r,\! \phi_r\!)\sin\!\big(\theta_r \!\!+\!\! \tau^{*\theta}\!(\theta_r,\!\phi_r\!)\!\big)\! \right]\!\! d\theta_r d\phi_r d\rho_r \nonumber \\
&=\frac{1}{8\pi^2}\!\int\! k(\theta_r, \phi_r, \rho_r, \theta_u, \phi_u) x(\theta_r, \phi_r) d\theta_r d\phi_r d\rho_r
\end{align}
where $u \!=\! (\theta_u,\phi_u)$ in the spherical coordinate system and $k(\theta_r, \phi_r, \rho_r, \theta_u, \phi_u)$ is the kernel of the operator $H- H {P^*_\tau}^{-1}$. Recall the Schur's Lemma \cite{schur1911bemerkungen}. For the operator $K$ with kernel $k(\theta_r, \phi_r, \rho_r, \theta_u, \phi_u)$, if 
\begin{align}\label{eq:prthm117}
\frac{1}{8\pi^2}\int | k(\theta_r, \phi_r, \rho_r, \theta_u, \phi_u) | d\theta_r d\phi_r d\rho_r \le A,\\
\label{eq:prthm1175}\frac{1}{2 \pi^2} \int | k(\theta_r, \phi_r, \rho_r, \theta_u, \phi_u) | d\theta_u d\phi_u \le A,
\end{align}
we have $\| K \| \le A$. We proceed by using the Schur's Lemma.

We divide $k(\theta_r, \phi_r, \rho_r, \theta_u, \phi_u)$ into two sub-operators as
\begin{align}\label{eq:prthm118}
&k(\theta_r, \!\phi_r,\! \rho_r,\! \theta_u,\! \phi_u) \!=\! k_{1}(\theta_r,\! \phi_r, \!\rho_r,\! \theta_u,\! \phi_u)\!+\!k_{2}(\theta_r,\! \phi_r, \!\rho_r,\! \theta_u,\! \phi_u)\\
&=\big(1-\alpha(\theta_r, \phi_r) \big)h\Big({r^z_{\rho_r}}^{-1} \big( r_{u_r}^{-1} (u)\big) \Big)\sin(\theta_r) \nonumber \\
&+\! \alpha(\theta_r, \phi_r) \!\left[h\Big(\!{r^z_{\rho_r}}^{-\!1} \big( r_{u_r}^{-\!1} (u)\big) \!\Big)\!\sin(\theta_r) \!-\!h \!\left(\!  {r^{z}_{\rho_r}}^{-1} \big( r_{\tau^* (u_r)}^{-1} (u)\big)\! \right)\!\sin(\theta_r \!+\! \tau^{*\theta}(\theta_r,\phi_r))\!\right]\!\!,\nonumber
\end{align}
and analyze these two sub-operators separately. 

For the first sub-operator $k_{1}(\theta_r, \phi_r, \rho_r, \theta_u, \phi_u)$, since $|\partial \tau^{*\theta}(\theta_r, \phi_r)/\partial \theta_r| \!\!\le \| \nabla\tau \|_\SO3$ and $|\partial \tau^{*\phi}(\theta_r, \phi_r)/\partial \phi_r| \le \| \nabla\tau \|_\SO3$ [cf. \eqref{eq:rotationSize}], we have
 \begin{gather} \label{eq:prthm119}
|1-\alpha(\theta_r, \phi_r)| \le \left( 2+\| \nabla\tau \|_\SO3 \right) \| \nabla\tau \|_\SO3.
\end{gather}
By substituting \eqref{eq:prthm119} into $k_{1}(\theta_r, \phi_r, \rho_r, \theta_u, \phi_u)$, we obtain
\begin{align} \label{eq:prthm120}
&| k_{1}(\theta_r, \phi_r, \rho_r, \theta_u, \phi_u)| \!\le\! |1\!-\!\alpha(\theta_r, \phi_r)| \Big| h\Big({r^z_{\rho_r}}^{-1} \big( r_{u_r}^{-1} (u)\big) \Big) \Big| \nonumber\\
& \le C_h \left( 2+\| \nabla\tau \|_\SO3 \right) \| \nabla\tau \|_\SO3 \le 2 C_h \epsilon + \ccalO(\epsilon^2)
\end{align}
where the last inequality uses the conditions that $h$ is a Lipschitz filter with respect to $C_h$ and $\| \nabla\tau \|_\SO3 \le \epsilon$.

For the second sub-operator $k_{2}(\theta_r, \phi_r, \rho_r, \theta_u, \phi_u)$, we have
\begin{equation}\label{eq:prthm121}
\begin{split}
|\alpha(\theta_r, \phi_r)| \le 1 + 2\| \nabla\tau \|_\SO3 +\| \nabla\tau \|_\SO3^2.
\end{split}
\end{equation}
Then consider the reduction term 
\begin{align}
h\Big({r^z_{\rho_r}}^{-\!1} \big( r_{u_r}^{-\!1} (u)\big)\Big)\sin(\theta_r) - h \Big( r^{z~-1}_{\rho_r} \big( r_{\tau^* (u_r)}^{-1} (u)\big) \Big)\sin\!\big(\theta_r + \tau^{*\theta}(\theta_r,\phi_r)\big) \nonumber
\end{align}
which can be divided into two sub-terms as
\begin{align}\label{eq:prthm1213}
& \sin(\theta_r)\!\left[ h\Big({r^z_{\rho_r}}^{-1} \big( r_{u_r}^{-1} (u)\big) \!\Big) \!-\! h \!\left( r^{z~-1}_{\rho_r} \big( r_{\tau^* (u_r)}^{-1} ( u) \big) \right) \right]\\
& +\! h\! \left(r^{z~-1}_{\rho_r} \big( r_{\tau^* (u_r)}^{-1} (u) \big) \!\right) \!\left[ \sin(\theta_r) \!-\! \sin\!\big(\theta_r \!+\! \tau^{*\theta}(\theta_r,\phi_r)\big) \right] \nonumber.
\end{align}
For the first term in \eqref{eq:prthm1213}, it is determined by two points ${r^z_{\rho_r}}^{-1} \big( r_{u_r}^{-1} (u)\big)$ and $r^{z~-1}_{\rho_r} \big( r_{\tau^* (u_r)}^{-1} ( u)\big)$. We start by considering two points $r_{\tau^* (u_r)}^{-1} (u)$ and $r_{u_r}^{-1} (u)$. Due to the symmetry of the sphere, we can alternatively consider two points $r_{\tau^* (u_r)} (u)$ and $r_{u_r} (u)$ such that
\begin{equation}\label{eq:prthm1214}
\begin{split}
r_{\tau^* (u_r)} (u) &= r_{\tau^* (u_r)} \Big( r^{-1}_{u_r} \big( r_{u_r} (u) \big)\!\Big) \!=\! \tau^*_{u_r} \big( r_{u_r} (u) \big)
\end{split}
\end{equation}
where $\tau^*_{u_r} = r_{\tau^* (u_r)} \circ r_{u_r}^{-1}$ is the local rotation of $\tau^*$ at the point $u_r$ [cf. \eqref{eq:rotationDiffeomorphism}], whose rotation angle is bounded by $\| \tau \|_\SO3$ [cf. \eqref{eq:rotationSize}]. The shortest arc connecting these two points on the sphere is then bounded by
\begin{equation}\label{eq:prthm1215}
\begin{split}
\text{arc}\big(r_{u_r} (u),~ r_{\tau^* (u_r)} (u)\big) \le \| \tau \|_\SO3,~ \forall u \in \mbS_2
\end{split}
\end{equation}
since points are defined on the unit sphere. Thus, due to the symmetry of the sphere, we have $\text{arc}\big(r_{u_r}^{-1} (u),~ r^{-1}_{\tau^* (u_r)} (u) \big) \le \| \tau \|_\SO3$ for all $u \in \mbS_2$. Since $r^{z~-1}_{\rho_r}$ is the rotation along $z$-axis that does not change the distance between two points on the sphere, we get
\begin{equation}\label{eq:prthm12165}
\begin{split}
\text{arc}\Big(r^{z~-1}_{\rho_r}\big( r_{u_r}^{-1}(u)\big),~\! r^{z~-1}_{\rho_r} \big( r^{-1}_{\tau^* (u_r)} (u)\big) \!\Big) \!\le\! \| \tau \|_\SO3,~\forall~ u \in \mbS_2.
\end{split}
\end{equation}
With the Lipschitz property of filter $h$, we get
\begin{align} \label{eq:prthm1217}
&\big| \sin(\theta_r) \left[ h\big(r_{u_r}^{-1} (u) \big) \!-\! h \big( r^{-1}_{\tau^* (u_r)} (u) \big)\right] \big| \nonumber\\
&\le\! C_h ~\!\text{arc}\Big(r^{z-1}_{\rho_r}\big( r^{-1}_{\phi_r\theta_r\rho_r} (u)\big),~\! r^{z-1}_{\rho_r}\big( r^{-1}_{\tau^* (u_r)} (u)\big) \!\Big) \!\le C_h \| \tau \|_\SO3.
\end{align}
For the second term in \eqref{eq:prthm1213}, by using the truncated Taylor Expansion, we can represent $\sin\big(\theta_r \!+\! \tau^{*\theta}(\theta_r,\phi_r)\big)$ as
\begin{align} \label{eq:prthm122}
&\sin\!\big(\theta_r \!+\! \tau^{*\theta}(\theta_r,\phi_r)\big) \!=\! \sin(\theta_r) \!+\! \cos\!\big(\theta_r+ t \tau^{*\theta}(\theta_r,\phi_r)\big) \tau^{*\theta}(\theta_r,\phi_r)
\end{align}
with some $t \in (0,1)$. Now denote by $u_1 = u_r = (\theta_r, \phi_r)$ and $u_2 = \tau^* (u_1)\!\!=\big(\theta_r+ \tau^{*\theta}(\theta_r,\phi_r), \phi_r+ \tau^{*\phi}(\theta_r,\phi_r)\big)$ two points on the sphere, and consider the triangle $\Delta u_1u_0u_2$ with $u_0=(0,0)$. Since the difference between two sides is less than the third side, we have $|\overline{u_0u_1} - \overline{u_0u_2}| \le \overline{u_1u_2}$. Therefore, we get
\begin{align} \label{eq:prthm124}
|\text{arc}(u_0,u_1) - \text{arc}(u_0,u_2)| \le \text{arc}(u_1,u_2).
\end{align}
Also since $u_2 = \tau^* (u_1) = \tau^*_{u_1} (u_1)$ where $\tau^*_{u_1}$ is the local rotation of $\tau^*$ at the point $u_1$ and $\text{arc}(u_1,u_2)$ is the shortest distance between $u_1$ and $u_2$, we have $\text{arc}(u_1,u_2) \le \| \tau \|_\SO3$. By using this result in \eqref{eq:prthm124} and the facts that $\text{arc}(u_0,u_1) = \theta_r$ and $\text{arc}(u_0,u_2) = \theta_r + \tau^{*\theta}(\theta_r,\phi_r)$ in the unit sphere, we have
\begin{align} \label{eq:prthm125}
|\theta_r + \tau^{*\theta}(\theta_r,\phi_r) - \theta_r| = |\tau^{*\theta}(\theta_r,\phi_r)| \le \| \tau \|_\SO3.
\end{align}
By substituting \eqref{eq:prthm125} into \eqref{eq:prthm122} and the latter into the second term in \eqref{eq:prthm1213} together with the fact that $|\cos\!\big(\theta_r+ t \tau^{*\theta}(\theta_r,\phi_r)\big)|\le 1$, we get
\begin{align} \label{eq:prthm126}
&\Big|h \left( r^{z-1}_{\rho_r}\big( r_{\tau^* (u_r)}^{-\!1} (u) \big) \right) \!\left[ \sin(\theta) \!-\! \sin\big(\theta_r \!+\! \tau^{*\theta}(\theta_r,\!\phi_r)\big) \right]\!\Big| \le C_h \| \tau \|_\SO3.
\end{align}
Further substituting \eqref{eq:prthm121}, \eqref{eq:prthm1217} and \eqref{eq:prthm126} into $k_{2}(\theta_r, \phi_r, \rho_r, \theta_u, \phi_u)$ yields
\begin{align}\label{eq:prthm127}
&| k_{2}(\theta_r, \phi_r, \rho_r, \theta_u, \phi_u) |\\
&\!\le\! 2 C_h \| \tau \|_\SO3\!\!\left(\! 1 \!+\! 2\| \nabla\tau \|_\SO3 \!+\!\| \nabla\tau \|_{\SO3}^2 \!\right) \nonumber \!\le\! 2 C_h \epsilon \!+\! \ccalO(\epsilon^2)
\end{align}
where $\| \tau \|_\SO3 \le \epsilon$ and $\| \nabla \tau \|_\SO3 \le \epsilon$ are used in the last inequality.

By using \eqref{eq:prthm120} and \eqref{eq:prthm127}, we bound the operator $k(\theta_r, \phi_r, \rho_r, \theta_u, \phi_u)$ as
\begin{align} \label{eq:prthm128}
| k(\theta_r, \phi_r, \rho_r, \theta_u, \phi_u)| \le 4 C_h \epsilon + \ccalO(\epsilon^2).
\end{align}
Therefore, we obtain
\begin{align}
\frac{1}{8\pi^2}\!\!\int\! | k(\theta_r,\! \phi_r,\! \rho_r,\! \theta_u,\! \phi_u) | d\theta_r d\phi_r d\rho_r \!\le\! 4 C_h \epsilon \!+\! \ccalO(\epsilon^2),\\
\frac{1}{2 \pi^2}\! \int\! | k(\theta_r, \phi_r, \rho_r, \theta_u, \phi_u) | d\theta_u d\phi_u \!\le\! 4 C_h \epsilon \!+\! \ccalO(\epsilon^2).
\end{align}
Then by using the Schur's Lemma, we have
\begin{align}\label{eq:prthm129}
&\| H \!-\! H {P^*_\tau}^{-\!1} \| \le 4 C_h \epsilon + \ccalO(\epsilon^2).
\end{align}

Finally, by substituting \eqref{eq:prthm18} and \eqref{eq:prthm129} into \eqref{eq:prthm15} and the latter into \eqref{eq:proofthm13}, we complete the proof
\begin{equation}\label{eq:prthm130}
\| H(x_\tau) \!-\! H(x) \|_{\SO3} \le \| HP^*_\tau \!-\! H \| \|\bbx\|
 \le 8 C_h \epsilon \| \bbx \| \!+\! \ccalO(\epsilon^2).
\end{equation}
\end{proof}

\section{Proof of Proposition~\ref{PropositionEquivarianceSCNN}} \label{proof:propositionEquivarianceSCNN}

\begin{proof}[Proof of Proposition 2]
Without loss of generality, we consider multi-feature Spherical CNN. From Proposition \ref{prop:rotationEquivarianceConv}, we know that the rotation equivariance holds for spherical convolutional filters. Then at layer $\ell$ of the SCNN, each filter $h_\ell^{fg}$ satisfies
\begin{equation} \label{eq:prproposition21}
\begin{split}
H_\ell^{fg}\left(\big(x_{(\ell-1)}^g\big)_r\right) = \left(H_{\ell}^{fg} (x_{\ell-1}^g)\right)_r,~\forall~f=1,\ldots,F_\ell,~ g=1,\ldots,F_{\ell-1},
\end{split}
\end{equation}
where $\big(x_{(\ell-1)}^g\big)_r$, $\big(H_{\ell}^{fg} (x_{\ell-1}^g)\big)_r$ are rotated signals of $x_{\ell\!-\!1}^g$, $H_{\ell}^{fg} (x_{\ell\!-\!1}^g)$ by $r$. Since linear operations does not break the rotation equivariance, we have
\begin{equation} \label{eq:prproposition22}
\begin{split}
\sum_{g=1}^F H_\ell^{fg}\left(\big(x_{(\ell-1)}^g\big)_r\right) \!=\! \Big( \sum_{g=1}^F H_\ell^{fg} (x_{\ell-1}^g) \Big)_r.
\end{split}
\end{equation}
Note that the pointwise nonlinearity $\sigma(\cdot)$ applies to each element of the spherical signal $x$ such that we have $\sigma(x_r) = \big(\sigma(x)\big)_r$. Thus, we can further get
\begin{equation} \label{eq:prproposition23}
\begin{split}
 &\sigma \Big( \sum_{g=1}^F H_\ell^{fg}\left(\big( x_{(\ell-1)}^g \big)_r \right) \Big) \!=\! \Big( \sigma \Big( \sum_{g=1}^F H_\ell^{fg} (x_{\ell-1}^g) \Big)\Big)_r.
\end{split}
\end{equation}
Since \eqref{eq:prproposition23} holds at each layer $\ell=1,\cdots,L$, we conclude that the rotation equivariance holds for the Spherical CNN.
\end{proof}


\section{Proof of Theorem~\ref{thm:stabilitySCNN}} \label{proof:stabilitySCNN}

We need the following lemma that shows the propagation consequence of a spherical signal through the spherical convolutional filter.

\begin{lemma}\label{LemmaFilterBound}
Let $x \in\mathbb{L}^2(\mbS_2) $ be a spherical signal, and $H$ be a Lipschitz spherical convolutional filter with respect to $C_h$ [cf. Def~\ref{def:Lipschitz}]. Then, it holds that
    \begin{gather}
    \| H( x) \| \le C_h \| x \|.
    \end{gather}
\end{lemma}
\begin{proof}
    The output of the spherical convolutional filter is [cf. \eqref{eq:sphericalConvRectangle}]
\begin{align} \label{eq:prlemma11}
H(x) (u) &=\!\frac{1}{8 \pi^2} \!\!\! \int\!\!
        h\big(r^{-1}_{\phi_r\theta_r\rho_r}(u)\big) x(\theta_r,\phi_r)\!\sin(\theta_r)
         d\theta_r d\phi_r d\rho_r\nonumber \\
& =\! \frac{1}{8 \pi^2} \!\!\! \int\!\!k(\theta_r,\!\phi_r,\!\rho_r,\!\theta_u,\!\phi_u) x(\theta_r,\!\phi_r\!)
         d\theta_r d\phi_r d\rho_r 
\end{align}    
where $r^{-1}_{\phi_r\theta_r\rho_r}$ is the inverse rotation characterized by the Euler angles $(\phi_r,\!\theta_r,\!\rho_r)$ and $k(\theta_r,\!\phi_r,\!\rho_r,\!\theta_u,\!\phi_u)$ is the operation kernel. Since $|h(\theta_r,\phi_r)| \le C_h$, we have
    \begin{align}
    &\frac{1}{8\pi^2}\!\!\int\!\! |k(\theta_r,\!\phi_r,\!\rho_r,\!\theta_u,\!\phi_u)| d\theta_r d\phi_r d\rho_r \!\le\! C_h,~\frac{1}{2\pi^2}\!\!\int\!\! |k(\theta_r,\!\phi_r,\!\rho_r,\!\theta_u,\!\phi_u)| d\theta_u d\phi_u \!\le\! C_h \nonumber
    \end{align}
    By using the Schur's Lemma in \eqref{eq:prthm117} and \eqref{eq:prthm1175}, we complete the proof $\| H( x) \| \le C_h \|x\|$.
\end{proof}

\begin{proof}
The output of the Spherical CNN is $\Phi(x;\ccalH) = \sigma_L \big( \sum_{f=1}^{F} h_{L}^{f} \ast_{\SO3} x_{L-1}^{f} \big)$. Denote by $H_L^f x_{L-1}^f \!=\! h_{L}^{f} \ast_{\SO3} x_{L-1}^{f}$ and $P_\tau x \!=\! x_\tau$ concise notions of the spherical convolution and the perturbed signal. The output difference induced by $\tau$ is
\begin{align} \label{eq:prthm22}
\| \Phi(P_\tau x;\ccalH) - \Phi(x;\ccalH) \| &= \big\| \sigma \big( \sum_{f=1}^{F}\! H_{L}^{f} \hat{x}_{L-1}^{f} \big) - \sigma \big( \sum_{f=1}^{F}\! H_{L}^{f} x_{L-1}^{f} \big)\big\| \nonumber \\
& \le C_\sigma \big\| \sum_{f=1}^{F} H_{L}^{f} \hat{x}_{L-1}^{f} - \sum_{f=1}^{F} H_{L}^{f} x_{L-1}^{f} \big\| 
\end{align}
where $\hat{x}_{L-1}^f$ and $x_{L-1}^f$ are the $f$th outputs of $\Phi(P_\tau x;\ccalH)$ and $\Phi(x;\ccalH)$ at $(L-1)$th layer, respectively. The last inequality is due to the Lipschitz property of the nonlinearity $\sigma(\cdot)$. Since the spherical convolutional filter is linear, we get
\begin{align} \label{eq:prthm23}
\!\!\!\| \Phi(P_\tau x;\!\ccalH) \!-\! \Phi(x;\!\ccalH) \| &\!\le\! \sum_{f=1}^{F}\!\! C_\sigma \big\| H_L^f \!\left(\! \hat{x}_{L\!-\!1}^{f} \!-\!  x_{L\!-\!1}^{f}\!\right)\! \big\| \!\le\! C_\sigma C_h\!\! \sum_{f=1}^{F}\! \| \hat{x}_{L\!-\!1}^{f} \!-\!  x_{L\!-\!1}^{f} \|
\end{align}
where the first inequality is due to the triangle inequality and the second is obtained from Lemma \ref{LemmaFilterBound}. We observe a recursive process in \eqref{eq:prthm23} that the output difference at $L$th layer is upper bounded by that at $(L-1)$th layer. By repeating this process, we have $\| \hat{x}_{L-1}^{f} -  x_{L-1}^{f} \| \le C_\sigma C_h \sum_{g=1}^{F} \| \hat{x}_{L-2}^{g} -  x_{L-2}^{g} \|$. Substituting this result into \eqref{eq:prthm23} yields
\begin{align} \label{eq:prthm25}
&\| \Phi(P_\tau x;\ccalH) \!-\! \Phi(x;\ccalH) \| \!\le\! C_\sigma^2 C_h^2 F \!\sum_{g=1}^{F} \| \hat{x}_{L-2}^{g} \!-\!  x_{L-2}^{g} \|.
\end{align}
By doing this recursively, we have
\begin{align} \label{eq:prthm26}
\| \Phi(P_\tau x;\ccalH) - \Phi(x;\ccalH) \| \le C_\sigma^{L-1}C_h^{L-1} F^{L-2} \sum_{g=1}^{F} \| \hat{x}_{1}^{g} \!-\!  x_{1}^{g} \|
\end{align}
where $\hat{x}_{1}^g$ and $x_{1}^g$ are the $g$th outputs of $\Phi(P_\tau x;\ccalH)$ and $\Phi(x;\ccalH)$ at $1$st layer.

Now consider the term $\| \hat{x}_{1}^{g} -  x_{1}^{g} \|$, which is given by the definition as
\begin{equation} \label{eq:prthm27}
\| \hat{x}_{1}^{g} \!-\!  x_{1}^{g} \| \!= \! \| \sigma \!\left( H_{1}^{g} P_\tau x \right) \!-\! \sigma\! \left( H_{1}^{g} x \right)\| \!\le \! C_\sigma \| H_{1}^{g} P_\tau x - H_{1}^{g} x \|
\end{equation}
where $x$ is the input signal and the last inequality is due to the Lipschitz property of the nonlinearity. From Theorem \ref{thm:stabilityConv}, under the rotation distance we have 
\begin{equation} \label{eq:prthm28}
\begin{split}
&\| H_{1}^{g} P_\tau x - H_{1}^{g} x \|_\SO3 \le 8 C_h \epsilon \| x \|+ \ccalO(\epsilon^{2}).
\end{split}
\end{equation}
Note that \eqref{eq:prthm28} holds for any filters $\{ h_{1}^{g} \}_{g=1}^F$ at $1$st layer. By substituting \eqref{eq:prthm28} into \eqref{eq:prthm27}, we get $\| \hat{x}_{1}^{g} -  x_{1}^{g} \|_\SO3 \le 8 C_h C_\sigma \epsilon \| x \|+ \ccalO(\epsilon^{2})$ for all $g=1,\ldots, F$. By further substituting this result into \eqref{eq:prthm26}, we complete the proof
\begin{equation} \label{eq:prthm210}
\| \Phi(P_\tau x;\ccalH) - \Phi(x;\ccalH) \|_\SO3
\le 8 C_h^{L} C_\sigma^{L} F^{L-1} \epsilon \| x \| + \ccalO(\epsilon^{2}).
\end{equation}
\end{proof}